\documentclass[twoside,11pt]{article}

\usepackage{jmlr2e}

\usepackage{amsmath}\allowdisplaybreaks
\usepackage{amsfonts,bm}
\usepackage{amssymb}
\usepackage{algorithm}
\usepackage{algorithmic}
\usepackage{multirow}
\usepackage{booktabs}
\usepackage{xcolor}

\def\eqref#1{equation~(\ref{#1})}

\def\1{\bf{1}}

\newcommand{\Norm}[1]{\left\| #1 \right\|}

\def\eps{{\varepsilon}}

\def\fA{{\mathcal{A}}}

\def\fN{{\mathcal{N}}}
\def\fO{{\mathcal{O}}}

\def\fX{{\mathcal{X}}}
\def\fY{{\mathcal{Y}}}

\def\BE{{\mathbb{E}}}

\def\BP{{\mathbb{P}}}

\def\BR{{\mathbb{R}}}

\def\Gap{{\rm{Gap}}}

\newtheorem{thm}{Theorem}[section]
\newtheorem{dfn}{Definition}[section]

\newtheorem{lem}{Lemma}[section]
\newtheorem{asm}{Assumption}[section]
\newtheorem{rmk}{Remark}[section]
\newtheorem{cor}{Corollary}[section]
\newtheorem{prop}{Proposition}[section]

\makeatletter
\def\Ddots{\mathinner{\mkern1mu\raise\p@
\vbox{\kern7\p@\hbox{.}}\mkern2mu
\raise4\p@\hbox{.}\mkern2mu\raise7\p@\hbox{.}\mkern1mu}}
\makeatother

\makeatletter
\newcommand*{\rom}[1]{\expandafter\@slowromancap\romannumeral #1@}
\makeatother

\usepackage{hyperref,url}

\def\RAIN {{\rm RAIN}}
\def\SEG {{\rm SEG}}

\def\ESEGp {{\rm Epoch-SEG$^+$}}
\def\RAINpp {{\rm RAIN$^{++}$}}
\def\SEGpp {{\rm SEG$^{++}$}}

\usepackage{lastpage}
\jmlrheading{25}{2024}{1-\pageref{LastPage}}{10/22}{12/24}{22-1126}{Lesi Chen and Luo Luo}
\ShortHeadings{Making the Gradient Small in Stochastic Minimax Optimization}{Chen and Luo}

\begin{document}

\title{Near-Optimal Algorithms for Making the Gradient Small in Stochastic Minimax Optimization}

\author{\name Lesi Chen \email chenlc23@mails.edu.cn \\
       \addr Institute for Interdisciplinary Information Sciences, 
       Tsinghua University \\
    Beijing, China 
       \AND
       \name Luo Luo\thanks{The corresponding author} \email luoluo@fudan.edu.cn \\
       \addr \addr School of Data Science, 
       Fudan University \\
       Shanghai, China \\
       and \\
       Shanghai Key Laboratory for Contemporary Applied Mathematics \\
       Shanghai, China}
\editor{Lam Nguyen}
\maketitle


\begin{abstract}
We study the problem of finding a near-stationary point for smooth minimax optimization. The recently proposed extra anchored gradient (EAG) methods achieve the optimal convergence rate for the convex-concave minimax problem in the deterministic setting.  However, the direct extension of EAG to stochastic optimization is not efficient.
In this paper, we design a novel stochastic algorithm called Recursive Anchored IteratioN (RAIN).  We show that the RAIN achieves near-optimal stochastic first-order oracle (SFO) complexity for stochastic minimax optimization in both convex-concave and strongly-convex-strongly-concave cases. In addition, we extend the idea of RAIN to solve structured nonconvex-nonconcave minimax problem and it also achieves near-optimal SFO complexity.
\end{abstract}


\section{Introduction}

This work studies the stochastic minimax problem of the form:
\begin{align}\label{prob:main}
    \min_{x \in \BR^{d_x}} \max_{y \in \BR^{d_y}} f(x,y) \triangleq \BE[f(x,y; \xi)],
\end{align}
where the stochastic component $f(x,y; \xi)$ is indexed by some random variable $\xi$; and the objective function $f(x,y)$ is $L$-smooth. 
This formulation has aroused great interest in machine learning community~\citep{lin2020gradient,liu2022initialization,zhang2021complexity,yang2020global,xu2020enhanced,xian2021faster,zhang2022bring,zhang2022uniform} due to its wide applications, including generative adversarial networks (GANs)~\citep{goodfellow2014generative,goodfellow2014explaining,liu2020decentralized}, AUC maximization~\citep{guo2020fast,yuan2021federated,liu2019stochastic,yang2022auc}, adversarial training~\citep{madry2017towards} and multi-agent reinforcement learning~\citep{omidshafiei2017deep,dai2018boosting,wai2018multi}. 

We focus on using stochastic first-order oracle (SFO) algorithms to find an $\epsilon$-stationary point of problem~(\ref{prob:main}), that is, a point $(x,y)$ where the gradient satisfies 
\begin{align*}
\|\nabla f(x,y)\|\leq\epsilon.     
\end{align*}
The norm of gradient is a generic metric to quantify the sub-optimality for solving smooth optimization problem. 
It is always well-defined for differentiable objective functions and helps us understand general minimax problems without convex-concave assumption.
Some recent works consider making the gradient small in minimax optimization~\citep{yoon2021accelerated,cai2022stochastic,diakonikolas2021efficient,lee2021fast}, but the understanding of this task in stochastic setting is still limited.

For convex-concave minimax problems, a popular metric to quantify an approximate solution $(x,y)$ on the domain $\fX\times\fY$ is the duality gap, which is defined as
\begin{align*}
\Gap(x,y)\triangleq \max_{y' \in \fY} f(x,y') - \min_{x' \in \fX} f(x',y).
\end{align*}
The optimal SFO algorithms in terms of duality gap have been established~\citep{zhao2022accelerated,alacaoglu2022stochastic},
while the duality gap is difficult to be measured and it is not always well-defined. 
In contrast, the norm of gradient is a more general metric in optimization problems and it is easy to be measured in practice.
\citet{yoon2021accelerated} proposed the extra anchor gradient (EAG) method for finding $\epsilon$-stationary point of deterministic (offline) convex-concave minimax problem. 
They showed that EAG is an optimal deterministic first-order algorithm to find an $\epsilon$-stationary point and its exact first-order oracle upper bound complexity is $\fO(L\epsilon^{-1})$.
For the stochastic convex-concave minimax problem, \citet{lee2021fast} proved the stochastic EAG (SEAG) can achieve an SFO complexity of $\fO(\sigma^2 L^2 \epsilon^{-4} + L \epsilon^{-1})$, where $\sigma^2$ is a bound on the variance of gradient estimates.
Recently, \citet{cai2022stochastic} proposed a variant of stochastic Halpern iteration, which has an SFO complexity of $\tilde \fO(\sigma^2 L \epsilon^{-3} + L^3 \epsilon^{-3})$ under the additional average-smooth assumption.
However, the tightness of SFO complexity for finding $\epsilon$-stationary points in stochastic convex-concave minimax optimization was hitherto unclear.

For nonconvex-nonconcave minimax problems, there exists the counter-example that all
first-order algorithms will diverge~\citep{lee2021fast}. 
Hence, we are required to introduce additional assumptions.
\citet{grimmer2020landscape} proposed the intersection dominant condition to relax the convex-concave assumption and established the convergence guarantee of the exact proximal point method in such a setting. 
Later, \citet{diakonikolas2021efficient} considered the negative comonotonicity condition as a relaxation of the common monotonicity property in the convex-concave setting. They proposed a variant of stochastic extragradient named as EG$^+$, which has SFO upper bound of $\fO(\sigma^2 L^2 \epsilon^{-4} + L^2 \epsilon^{-2})$ for finding an $\epsilon$-stationary point, matching the complexity of ordinary stochastic extragradient (SEG) in convex-concave case. 
In addition, \citet{lee2021fast} extend the idea of EAG to solve minimax problem under the negative comonotonicity condition, but their analysis does not contain the stochastic algorithms.  

In this paper, we propose a novel stochastic algorithm called Recursive Anchored IteratioN (RAIN) to make the gradient small in stochastic minimax optimization. The algorithm solves a sequence of anchored sub-problems by the variant of the stochastic extragradient method, resulting the SFO upper bounds complexity of $\tilde \fO(\sigma^2 \epsilon^{-2} + L \epsilon^{-1})$ and $\tilde \fO(\sigma^2 \epsilon^{-2} + \kappa)$ for finding an $\epsilon$-stationary point in convex-concave and strongly-convex-strongly-concave settings respectively, where $\kappa$ is the condition number.
We also provide lower bounds to show the optimality of RAIN. We summarize our results for the convex-concave problem and compare them with prior work in Table~\ref{tab:result}. 
Additionally, we extend the idea of RAIN and propose the algorithm for solving stochastic nonconvex-nonconcave minimax problems. We prove that the SFO upper bound complexity of RAIN$^{++}$ is near-optimal under both intersection dominant~\citep{grimmer2020landscape} and negative comonotonic conditions~\citep{diakonikolas2021efficient}. We present the results for these structured nonconvex-nonconcave minimax in Table \ref{tab:NC}.
To the best of our knowledge, RAIN is the first near-optimal SFO algorithm for finding near-stationary points of the stochastic convex-concave minimax problem.

\begin{table}[t]
    \centering
    \caption{We summarize the SFO complexities for finding an $\epsilon$-stationary point in convex-concave setting. We denote SCSC and CC as $\lambda$-strongly-convex-$\lambda$-strongly-concave case and the general convex-concave case respectively. In the case of SCSC, we define the condition number as $\kappa \triangleq L/\lambda$. {
    We also define $D \triangleq \Vert z_0 - z^\ast \Vert$, where~$z_0$ is the initial point of the algorithm.
    The dependency on $D$ in the complexity of PDHG and RAIN (the case of SCSC) is only contained in the logarithmic term, which is omitted by using the notation $\tilde \fO(\,\cdot\,)$.}
    The results denoted by $^*$ indicate the analysis~\cite{cai2022stochastic} requires the additional assumption of average smoothness.}
    \label{tab:result}
    \renewcommand{\arraystretch}{1.2}
    \begin{tabular}{cccc}
    \hline
     Setting  & Algorithm     & Complexity  & Reference \\
    \hline
    \multirow{4}*{SCSC} & Halpern &  $\tilde \fO(\lambda^{-2} \kappa \sigma^2  \epsilon^{-2} + \kappa^3 D^2 \epsilon^{-2})^*$ & \cite{cai2022stochastic} \\
    ~ & PDHG & $\tilde \fO(\kappa  \sigma^2 \epsilon^{-2} + \kappa)$ &  \cite{zhao2022accelerated} \\
     ~ & RAIN & $\tilde \fO(\sigma^2 \epsilon^{-2} + \kappa )$ & Theorem \ref{thm:RAIN} \\
     ~ & Lower Bound & $\tilde \Omega(\sigma^2 \epsilon^{-2}  + \kappa)$ & Theorem \ref{thm:lower-SC} \\
    \hline
    \multirow{5}*{CC}
      & SEG  & $\fO(\sigma^2 L^2 \epsilon^{-4} + L^2 D^2 \epsilon^{-2})$  & \cite{diakonikolas2021efficient} \\
     ~ & SEAG  & $\fO(\sigma^2 L^2 \epsilon^{-4} + L D \epsilon^{-1})$  & \cite{lee2021fast} \\
     ~ & Halpern  &  $\tilde \fO(\sigma^2 L D \epsilon^{-3} + L^3 D^3 \epsilon^{-3})^*$ & \cite{cai2022stochastic} \\ 
     ~ & RAIN  & $\tilde \fO(\sigma^2 \epsilon^{-2} + L D  \epsilon^{-1})$  & Theorem \ref{thm:EA-C} \\
     ~ & Lower Bound & $\tilde \Omega(\sigma^2 \epsilon^{-2} + L D \epsilon^{-1})$ & Theorem \ref{thm:lower-C} \\
    \hline
    \end{tabular}
\end{table}

\section{Notation and Preliminaries}
We use $\Norm{\,\cdot\,}$ to present the Euclidean norm of matrix and vector respectively. 
For differentiable function $f(x,y)$, we denote $\nabla_x f(x, y)$ and $\nabla_y f(x, y)$ as its partial gradients with respect to $x$ and $y$
respectively, and introduce the gradient operator 
\begin{align*}
F(x,y) \triangleq \begin{bmatrix}
\nabla_x f(x,y) \\ - \nabla_y f(x,y)
\end{bmatrix}.
\end{align*}
For ease of presentation, we define $z=(x,y)$ and also write the gradient operator at $z=(x,y)$ as $F(z)$.

\begin{table}[t]
    \caption{We summarize the SFO complexities for finding an  $\epsilon$-stationary point in two specific nonconvex-nonconcave settings. 
    We denote NC as the $\rho$-comonotonicity assumption with $\rho \in [-\frac{1}{2L},0 )$ and denote ID as the $(\tau,\alpha)$-intersection-dominant assumption with $\alpha >0$ and $\tau \ge 2L$. 
    {We also define $D \triangleq \Vert z_0 - z^\ast \Vert$, where~$z_0$ is the initial point of the algorithm.
    The dependency on $D$ in the complexity of RAIN (the case of ID) is only contained in the logarithmic term, which is omitted by using the notation $\tilde \fO(\,\cdot\,)$.}} \label{tab:NC}
    \centering
    \renewcommand{\arraystretch}{1.2}
    \begin{tabular}{c c c c}
    \hline
    Setting   &  Algorithm  & Complexity  & Reference \\
    \hline
    \multirow{3}*{NC}     &  SEG${}^+$ & $\fO(\sigma^2 L^2 \epsilon^{-4} + L^2 D^2 \epsilon^{-2})$ & \cite{diakonikolas2021efficient}  \\
    ~ & RAIN${}^{++}$ & $ \tilde \fO(\sigma^2 \epsilon^{-2} + L D \epsilon^{-1})$ & Theorem \ref{thm:RAIN++-NC} \\
    ~ & Lower Bound & $\tilde{\Omega}(\sigma^2 \epsilon^{-2} + L D \epsilon^{-1})$ & Corollary \ref{cor:lower-NC} \\
    \hline
    \multirow{2}*{ID} & RAIN${}^{++}$ & $ \tilde \fO(\sigma^2 \epsilon^{-2} + L/\alpha)$ & Theorem \ref{thm:RAIN++-ID} \\
    ~ & Lower Bound & $ \tilde \Omega(\sigma^2 \epsilon^{-2} + L/\alpha)$ & Corollary \ref{cor:lower-ID} \\
    \hline
    \end{tabular}
\end{table}

We consider the following assumptions for our stochastic minimax problems.

\begin{asm}[stochastic first-order oracle] 
We suppose the stochastic first-order oracle $F(z;\xi)$ is unbiased and has bounded variance such that $\BE[F(z;\xi)] = F(z)$  and $\mathbb{E}\Vert F(z;\xi) - F(z) \Vert^2 \le \sigma^2$
for all $z=(x,y)$ and random index $\xi$.
\end{asm}

\begin{asm}[smoothness]\label{asm:smooth} We suppose the function $f(x,y)$ is $L$-smooth, that is there exists some $L>0$ such that
\begin{align*}
\Vert F(z) - F(z') \Vert \le L \Vert z - z' \Vert    
\end{align*}
for all $z=(x,y)$ and $z'=(x',y')$.
\end{asm}

\begin{asm}[convex-concave]\label{asm:cc}
We suppose the function $f(x,y)$ is convex-concave (CC), that is the function $f(\,\cdot\,,y)$ is convex for any given $y$ and the function $f(x,\,\cdot\,)$ is concave for any given $x$.
\end{asm}

\begin{asm}[strongly-convex-strongly-concave]\label{asm:scsc}
We suppose the function $f(x,y)$ is $\lambda$-strongly-convex-$\lambda$-concave for some $\lambda>0$, that is the function $f(x,y) + \frac{\lambda}{2} \Vert x \Vert^2 - \frac{\lambda}{2} \Vert y \Vert^2$ is convex-concave.
\end{asm}
\noindent For smooth and convex-concave objective function, the corresponding gradient operator $F$ has the monotonicity properties~\cite{rockafellar1970monotone} as follows. 
\begin{lem}[monotonicity]\label{lem:mon}
Under Assumption~\ref{asm:smooth} and~\ref{asm:cc}, it holds that
\begin{align*}
(F(z) - F(z'))^\top (z - z') \ge 0     
\end{align*}
for all $z=(x,y)$ and $z'=(x',y')$.
\end{lem}
\begin{lem}[strong monotonicity]\label{lem:smon}
Under Assumption~\ref{asm:smooth} and~\ref{asm:scsc}, it holds that
\begin{align*}
(F(z) - F(z'))^\top (z - z') \ge \lambda \Vert z - z' \Vert^2
\end{align*}
for all $z=(x,y)$ and $z'=(x',y')$.
\end{lem}

We are interested in finding an $\epsilon$-stationary point of the differentiable function $f(x,y)$, that is the point where the norm of its gradient (gradient operator) is small.
\begin{dfn}[nearly-stationary point]
We say $\hat z=(\hat x, \hat y)$ is an $\epsilon$-stationary point of the differentiable function $f(x,y)$ if it satisfies $\Vert \nabla f(\hat z) \Vert \le \epsilon$, or equivalently $\Vert F(\hat z ) \Vert \le \epsilon$.
\end{dfn}
\noindent Throughout this paper, we always assume there exists a stationary point for $f(x,y)$ and the initial point $z_0=(x_0,y_0)$ of the considered algorithm is in a bounded set.
\begin{asm} We suppose there exists some $z^*=(x^*,y^*)$ such that $\nabla f(z^*) = 0$, or equivalently, $F(z^*) = 0$.
\end{asm}
\begin{asm} \label{asm:bounded-init} We suppose the initial point $z_0=(x_0,y_0)$ satisfies $\Vert z_0 - z^* \Vert  \le D$ for some $D>0$, where $z^*=(x^*,y^*)$ is a stationary point of $f(x,y)$.
\end{asm}

\section{The Recursive Anchored Iteration}\label{sec:RAIN}

In this section, we focus on stochastic minimax optimization in the convex-concave setting.
We propose the Recursive Anchored InteratioN (RAIN) method in Algorithm \ref{alg:RAIN}. The RAIN calls the subroutine Epoch-SEG (Algorithm~\ref{alg:Epoch-SEG}) to find the point $z_{s+1} = (x_{s+1},y_{s+1})$ that is an approximate solution of the two-sided regularized minimax problem 
\begin{align}\label{prob:sub}
  \min_{x\in\BR^{d_x}}\max_{y\in\BR^{d_y}} f^{(s)}(x,y),
\end{align}
where $f^{(s)}(x,y)$ is defined as follows
\begin{align}\label{dfn:fs}
f^{(s)}(x,y) \triangleq \begin{cases}
f(x,y), & s=0, \\
\displaystyle{f^{(s-1)}(x,y)+\frac{\lambda_s}{2}\|x-x_s\|^2-\frac{\lambda_s}{2}\|y-y_s\|^2}, & s\geq 1.
\end{cases}
\end{align}
We call the sequence $\{(x_s,y_s)\}_{s=1}^S$ as the anchors of RAIN, which plays an important role in our convergence analysis. 

For strongly-convex-strongly-concave objective function $f(x,y)$, the corresponding gradient operator $F(x,y)$ is strongly-monotone and the output of RAIN has the following property.
\begin{lem}[recursively anchoring lemma]\label{lem:RAL}
Suppose the function $f(x,y)$ is $L$-smooth and $\lambda$-strongly-convex-$\lambda$-strongly-concave. Let $z_s^*$ be the unique solution to sub-problem (\ref{prob:sub}), then the output of \RAIN~(Algorithm~\ref{alg:RAIN}) with $\gamma = 1$ holds that
\begin{align*}
    \Vert F(z_S) \Vert \le 16 \lambda \gamma \sum_{s=1}^{S} (1+ \gamma)^{s-1} \Vert z_{s-1}^* - z_s \Vert.
\end{align*}
\end{lem}

\begin{rmk} We directly set $\gamma = 1$ for the simplification of the proof. In fact, setting $\gamma$ to be any positive constant would lead to the upper bound of the same order. 
\end{rmk}
\noindent Lemma~\ref{lem:RAL} indicates we can make $\|F(z_S)\|$ small if the subroutine provide $z_s$ that is sufficiently close to $z_{s-1}^*$ at each round.
The recursive definition of $f_s(x,y)$ means the condition numbers of the sub-problems is decreasing with the iteration of RAIN.
Hence, achieving an accurate solution to sub-problem (\ref{prob:sub}) will not be too expensive for large $s$.
We will show that the total SFO complexity of RAIN could nearly match the lower bound by designing the sub-problem solver carefully.

\begin{algorithm*}[t]  
\caption{RAIN $(f, z_0, \lambda, L,\{N_s\}_{s=0}^{S-1}, \{K_s \}_{s=0}^{S-1}, \gamma)$} 
\begin{algorithmic}[1] \label{alg:RAIN}
\STATE $f_0 \triangleq f,   \lambda_0 = \lambda  \gamma,   S = \lfloor \log_{(1+\gamma)}(L/\lambda) \rfloor$ \\[0.1cm]
\STATE \textbf{for} $s = 0,1,\cdots, S-1$ \\[0.1cm]
\STATE \quad $z_{s+1} \leftarrow \text{Epoch-SEG}(f_s,z_s,\lambda_s,2L,N_s,K_s)$ \\[0.1cm]
\STATE \quad $\lambda_{s+1} \leftarrow  (1+ \gamma)\lambda_s$ \\[0.1cm]
\STATE \quad $f_{s+1}(z) \triangleq f_s(z) + \frac{\lambda_{s+1}}{2} \Vert x - x_{s+1} \Vert^2 - \frac{\lambda_{s+1}}{2} \Vert y - y_{s+1} \Vert^2$ \\[0.1cm]
\STATE \textbf{return} $z_{S}$ 
\end{algorithmic}
\end{algorithm*}

For general convex-concave objective function $f(x,y)$, we use the initial point $z_0=(x_0,y_0)$ as the additional anchor and then we apply RAIN to find an $\epsilon$-stationary point of the following strongly-convex-strongly-concave function 
\begin{align}\label{dfn:g}
g(x,y) \triangleq f(x,y) + \frac{\lambda}{2}\|x-x_0\|^2 - \frac{\lambda}{2}\|y-y_0\|^2.
\end{align}
We denote the gradient operator of $g(x,y)$ as $G(x,y)$. Then
the following lemma provides the connection between the norms of $F(x,y)$ and $G(x,y)$.
\begin{lem}[anchoring lemma]\label{lem:anchoring}
Suppose the function $f(x,y)$ is smooth and convex-concave. We define $g(x,y) \triangleq f(x,y) + \frac{\lambda}{2}\|x-x_0\|^2 - \frac{\lambda}{2}\|y-y_0\|^2$ for some $(x_0,y_0)$ and denote its gradient operator as
\begin{align*}
G(x,y) \triangleq \begin{bmatrix}
\nabla_x g(x,y) \\ - \nabla_y g(x,y)
\end{bmatrix},
\end{align*}
then it holds that
\begin{align*}
    \Vert F(\tilde z) \Vert &\le 2 \Vert G(\tilde z) \Vert + \lambda \Vert z_0-z^* \Vert,
\end{align*}
for any $\tilde z=(\tilde x, \tilde y)$, where $z^*=(x^*,y^*)$ is the stationary point of $f(x,y)$ and $z_0=(x_0,y_0)$.
\end{lem}
\noindent According to Assumption~\ref{asm:bounded-init} and Lemma~\ref{lem:anchoring}, setting $\lambda = \Theta(\epsilon/D)$ leads to finding an $\epsilon$-stationary point of $f(x,y)$ can be reduced into finding an $\fO(\epsilon)$-stationary point of $g(x,y)$, which can be done by applying RAIN (Algorithm \ref{alg:RAIN}) on the strongly-convex-strongly-concave function $g(x,y)$. 

\paragraph{Connection to Related Work}
We provide some discussion on comparing RAIN with related work.
\begin{itemize}
\item In convex optimization, \citet{allen2018make} proposed the recursive regularization technique for finding the nearly stationary point stochastically, which can be regarded as a special case of our anchored framework. However, \citeauthor{allen2018make}'s analysis depends on the convexity of the objective function, which is not suitable for minimax problem. 
In contrast, the analysis of RAIN is mainly based on the monotonicity of the gradient operator, which is more general than convex optimization. 
\item In minimax optimization, \citet{yoon2021accelerated,lee2021fast} considered variants of (stochastic) extragradient method by using initial point $z_0$ as the fixed anchor. 
In contrast, the proposed algorithm RAIN adjusts the anchoring point $z_s$ with iterations, which leads to the sequence of anchoring points $\{z_s\}$ converges to $z^*$. 
As a result, the RAIN achieves near-optimal SFO complexity in the stochastic setting (see Table \ref{tab:result}).
\item Several existing methods~\citep{lin2020near,zhao2022accelerated,kovalev2022first,luo2021near,yang2020catalyst} introduce the proximal point iteration 
\begin{align*}
    (x_{s+1},y_{s+1}) \approx \arg\min_{x\in\BR^{d_x}}\max_{y\in\BR^{d_y}} f(x,y) + \frac{\beta}{2}\|x-x_s\|^2,
\end{align*}
which is useful to establish the near-optimal algorithms for unbalanced minimax optimization in the offline scenario. 
However, it is questionable whether the one-sided regularization is helpful in finding near-stationary points in stochastic minimax problem~(\ref{prob:main}). 
\end{itemize}


\section{Complexity Analysis for RAIN}\label{sec:convex-concave}

In this section, we analyze the sub-problem solver Epoch Stochastic ExtraGradient (Epoch-SEG) and show our RAIN has near-optimal SFO upper bound for finding $\epsilon$-stationary point of stochastic convex-concave minimax problem. 

The procedure of Epoch-SEG (Algorithm \ref{alg:Epoch-SEG}) depends on the ordinary stochastic extragradient method (SEG, Algorithm~\ref{alg:SEG}), which has the following property.
\begin{lem}[SEG]\label{lem:SEG}
Suppose the function $f(x,y)$ is $L$-smooth and $\lambda$-strongly-convex-$\lambda$-strongly-concave, and the SFO $F(x,y;\xi)$ is unbiased and has variance bounded by $\sigma^2$.  
Then \SEG (Algorithm~\ref{alg:SEG}) holds that
\begin{align} \label{eq:SEG-step}
\lambda \mathbb{E}\Vert z_{t+1/2} - z^* \Vert^2 \le \frac{1}{\eta} \mathbb{E}\big[\Vert z_{t+1} - z^* \Vert^2 - \Vert z_t - z^* \Vert^2\big]  + 16 \eta \sigma^2
\end{align}
for any $0 < \eta < 1/(4L)$.
\end{lem}
\noindent Taking the average on (\ref{eq:SEG-step}) over $t=0,\dots,T-1$ and applying Lemma~\ref{lem:smon}, we know the output of SEG satisfies
\begin{align*}
    \mathbb{E}\Vert \bar z - z^* \Vert^2 \le \frac{1}{\lambda \eta T} \mathbb{E}\Vert z_0 - z^* \Vert^2 + \frac{16 \eta \sigma^2}{\lambda},
\end{align*}
which means SEG is able to decrease the distance from the output $\bar z$ to the optimal solution $z^*$ with iterations. However, it only converges to a neighborhood of $z^*$ by using the fixed stepsize.

\begin{algorithm*}[t]  
\caption{Epoch-SEG $(f, z_0, \lambda, L, N, K)$} 
\begin{algorithmic}[1] \label{alg:Epoch-SEG}
\STATE \textbf{for} $k = 0,1,\cdots,N-1$ \textbf{do}  \\[0.1cm]
\STATE \quad $z_{k+1} \leftarrow \text{SEG}\big(f,z_{k},\frac{1}{4L}, \frac{8L}{\lambda}\big)$ \\[0.1cm]
\STATE \textbf{for} $k = N,N+1,\cdots,N+K-1$ \textbf{do} \\[0.1cm]
\STATE \quad $z_{k+1} \leftarrow \text{SEG}\big(f,z_{k},\frac{1}{2^{k-N+3}L }, \frac{2^{k-N+5}L}{\lambda}\big)$ \\[0.1cm]
\STATE \textbf{return} $z_{N+K}$.
\end{algorithmic}
\end{algorithm*}

\begin{algorithm*}[t]  
\caption{SEG $(f, z_0, \eta, T)$} 
\begin{algorithmic}[1] \label{alg:SEG}
\STATE \textbf{for} $t = 0,1,\cdots,T-1$ \textbf{do} \\[0.1cm]
\STATE \quad $\xi_i \leftarrow $ a random index \\[0.1cm]
\STATE \quad $z_{t+1/2} \leftarrow z_t - \eta F(z_t;\xi_i)$ \\[0.1cm]
\STATE \quad $\xi_j \leftarrow $ a random index \\[0.1cm]
\STATE \quad $z_{t+1} \leftarrow z_t - \eta F(z_{t+1/2};\xi_j)$ \\[0.1cm]
\STATE \textbf{return} $\bar z$ by uniformly sampling from $\{z_{t+1/2} \}_{t=0}^{T-1}$ 
\end{algorithmic}
\end{algorithm*}

Then we consider the epoch stochastic extragradient (Epoch-SEG, Algorithm~\ref{alg:Epoch-SEG}), which consists of two phases and each of them calling SEG as subroutine by different parameters.
The Epoch-SEG targets to reduce both the optimization error and statistical error in the iterations:
\begin{itemize}
\item In the first phase, we call SEG by fixed stepsize and fixed iteration numbers to decrease the optimization error, which is related to the number $\kappa \triangleq L / \lambda$ and the distance $\Vert z_0 - z^* \Vert$. 
\item In the second phase, the statistical error aroused from the variance of stochastic oracle has accumulated. Hence, we call SEG by decreasing stepsizes and increasing iteration numbers to reduce the statistical error.
\end{itemize}
The formal theoretical guarantee of Epoch-SEG is shown in Lemma~\ref{lem:Epoch-SEG}.

\begin{lem}[Epoch-SEG] \label{lem:Epoch-SEG}
Suppose the function $f(x,y)$ is $L$-smooth and $\lambda$-strongly-convex-$\lambda$-strongly-concave, and the stochastic first-order oracle  $F(x,y;\xi)$ is an unbiased estimator of $F(x,y)$ and has variance bounded by $\sigma^2$. 
Then {\rm Epoch-SEG} (Algorithm \ref{alg:Epoch-SEG}) holds that
\begin{align*}
\mathbb{E}\Vert z_{N+K} - z^* \Vert^2 \le \underbrace{ \frac{1}{2^{N+2K}} \mathbb{E}\Vert z_0 - z^* \Vert^2}_{\text{optimization error}} + \underbrace{\frac{8 \sigma^2}{2^K \lambda L}.}_{\text{statistical error}}
\end{align*}
Additionally, the total number of SFO calls is no more than 
$16\kappa N + 2^{K+6}\kappa$, where $\kappa=L/\lambda$.
\end{lem}

Combining the above results, we obtain the SFO upper bound complexity of RAIN (Algorithm~\ref{alg:RAIN}) for the strongly-convex-strongly-concave case. 
\begin{thm}[RAIN, SCSC] \label{thm:RAIN}
Suppose the function $f(x,y)$ is $L$-smooth and $\lambda$-strongly-convex-$\lambda$-strongly-concave, and the stochastic first-order oracle  $F(x,y;\xi)$ is an unbiased estimator of $F(x,y)$ and has variance bounded by $\sigma^2$. 
If we run {\rm RAIN} (Algorithm~\ref{alg:RAIN}) with 
\begin{align} \label{eq:para-EA}
\begin{split}
N_0 = \begin{cases}
\left \lceil \log_2 \left(\dfrac{512 \lambda^2 S^2 D^2}{\epsilon^2} \right ) \right \rceil, & s=0, \\
3, & s\geq 1,
\end{cases}
\qquad\text{and}\qquad
K_s = \left \lceil \log_2 \left(\frac{2048 \lambda_s S^2 \sigma^2}{L \epsilon^2} \right)\right \rceil,
\end{split}
\end{align}
then the output $z_S$ satisfies $\mathbb{E}\Vert F(z_S) \Vert\le \epsilon$ and the total number of SFO calls is no more than
\begin{align*}
    \mathcal{O} \left( \frac{L}{\lambda} + \frac{L}{\lambda} \log \left( \frac{\lambda D }{\epsilon} \log \left( \frac{L}{\lambda}\right) \right) + \frac{\sigma^2}{\epsilon^2} \log^3 \left( \frac{L}{\lambda}\right)\right).
\end{align*}
\end{thm}

\noindent We can find nearly stationary points for the general convex-concave case by introducing the regularized function $g(x,y)$ defined in (\ref{dfn:g}). Applying Theorem~\ref{thm:RAIN} and Theorem \ref{lem:anchoring}, we achieve the SFO upper bound complexity as follows.

\begin{thm}[RAIN, CC] \label{thm:EA-C}
Suppose the function $f(x,y)$ is $L$-smooth and convex-concave, and the stochastic first-order oracle $F(x,y;\xi)$ is an unbiased estimator of $F(x,y)$ and has variance bounded by $\sigma^2$. Running {\rm RAIN} (Algorithm \ref{alg:RAIN}) on function
\begin{align*}
g(x,y) \triangleq f(x,y) + \frac{\lambda}{2} \Vert x - x_0 \Vert^2 - \frac{\lambda}{2} \Vert y - y_0 \Vert^2     
\end{align*}
with $\lambda = \min \left\{ \epsilon/D, L \right\}$ outputs a $3\eps$-stationary point in expectation, and the total number of SFO calls is no more than
\begin{align*}
    \mathcal{O} \left( \frac{L D}{\epsilon} + \frac{L D}{\epsilon} \log \log \left( \frac{L D}{\epsilon} \right)+ \frac{\sigma^2}{\epsilon^2} \log^3 \left(\frac{L D}{\epsilon} \right) \right).
\end{align*}
\end{thm}
\noindent The comparison in Table~\ref{tab:result} shows both the results of Theorem~\ref{thm:RAIN} and Theorem~\ref{thm:EA-C} are better than existing algorithms and nearly match the lower bound.
Additionally, the theoretical results in this section are not limited to convex-concave minimax optimization. 
In fact, our analysis is mainly based on the Lipschitz continuity and the monotonicity of the gradient operator, which means the results are also applicable to more general problems of variational inequality with Lipschitz continuous and monotone operator~\citep{rockafellar1970monotone}.

\section{Extension to Nonconvex-Nonconcave Settings}\label{sec:extension}
In this section, we extend RAIN to solve nonconvex-nonconcave minimax problems. 
We focus on two settings of comonotone and intersection dominant conditions.
The comonotonicity is defined as follows.
\begin{dfn}[comonotonicity]\label{dfn:com}
We say the operator $F(\cdot)$ is $\rho$-comonotone if there exists some $\rho\in\BR$ such that 
\begin{align*}
    (F(z) - F(z'))^\top (z - z') \ge \rho\Vert F(z) - F(z') \Vert^2
\end{align*}
for all $z$ and $z'$.  We also say $F(\cdot)$ is a negative comonotonic operator when $\rho <0$.
\end{dfn}
\begin{rmk}
In the case of $\rho=0$, the $\rho$-comonotonicity reduces to monotonicity shown in Lemma ~\ref{lem:mon}. In the case of $\rho<0$, the $\rho$-comonotone gradient operator allow the objective function be nonconvex-nonconcave. 
Typically, we additionally require $\rho\geq -\fO(1/L)$ in the convergence analysis for optimization algorithms in the negative comonotone setting~\citep{diakonikolas2021efficient,lee2021fast}.
\end{rmk}

\noindent The intersection dominant condition~\citep{grimmer2020landscape} also allows the objective function be nonconvex-nonconcave, which is defined as follows.
\begin{dfn}[intersection dominant condition] For a twice differentiable function $f(x,y)$, we say it satisfies the $(\tau,\alpha)$-intersection-dominant condition if there exist some $\tau>0$ and $\alpha>0$ such that
\begin{align*}
\nabla_{xx}^2 f(x,y) +  \nabla_{xy}^2 f(x,y)(\tau I - \nabla_{yy}^2 f(x,y))^{-1} \nabla_{yx}^2 f(x,y) \succeq \alpha I 
\end{align*}
and
\begin{align*}
-\nabla_{yy}^2 f(x,y) +  \nabla_{yx}^2 f(x,y)(\tau I + \nabla_{xx}^2 f(x,y))^{-1} \nabla_{xy}^2 f(x,y) \succeq \alpha I
\end{align*}
for all $(x,y)$.
\end{dfn}

\begin{rmk}
Typically, we also require $\tau\geq\Omega(L)$ in conditions of intersection dominant for the analysis for minimax optimization~\citep{grimmer2020landscape,lee2021fast}.
\end{rmk}

\begin{rmk}
For smooth minimax optimization, the intersection dominant condition is stronger than negative comonotonicity. Concretely, the intersection dominant condition needs the function $f(x,y)$ to be twice differentiable, which is unnecessary to negative comonotonicity. Furthermore, if the function $f(x,y)$ is $L$-smooth and satisfies the $(\tau,\alpha)$-intersection-dominant condition for some $\alpha>0$ and $\tau >L$, then its gradient operator $F(\cdot)$ must satisfy the $-1/\tau$-comonotonic condition (see Example 1 of \citet{lee2021fast}).
\end{rmk}

It turns out that both of these two conditions are related to the saddle envelope, which is a natural generalization of Moreau envelope for minimax problems.
\begin{dfn}[saddle envelope]
Given some $\tau >0$, we define the saddle envelope of function $f(x,y)$ as 
\begin{align*}
    f_{\tau}(x,y) \triangleq \min_{x' \in \BR^{d_x}} \max_{y' \in \BR^{d_y}} f(x',y') + \frac{\tau}{2} \Vert x' - x \Vert^2 - \frac{\tau}{2} \Vert y' - y \Vert^2.
\end{align*}
\end{dfn}
\noindent The saddle envelope has the following properties.
\begin{prop}\label{prop:id-scsc}
Suppose function $f(x,y)$ is $L$-smooth and satisfies the $(\tau,\alpha)$-intersection dominant condition for $\tau \ge 2L$, then $f_{2L}(x,y)$ is $\lambda$-strongly-convex-$\lambda$-strongly-concave, where 
\begin{align} \label{rela:alpha-mu}
    \lambda = \left( \frac{1}{2L} + \frac{1}{\alpha} \right)^{-1} = \Theta(\alpha).
\end{align}
\end{prop}

\begin{prop}\label{prop:cc-envelope}
Suppose the function $f(x,y)$ is $L$-smooth and convex-concave, and its gradient operator $F(x,y)$ is $-\rho$-comonotonic with $0 < \rho \le 1/(2L)$, then its saddle envelope $f_{2L}(x,y)$ is convex-concave.
\end{prop}

Based on the observation that the stationary points of $f_{2L}(x,y)$ are exactly the same as those of $f(x,y)$~\cite[Corollary 2.2]{grimmer2020landscape}, it is possible to apply RAIN on the saddle envelope $f_{2L}(x,y)$ for nonconvex-nonconcave minimax optimization. 
However, accessing the (stochastic) gradient operator of $f_{2L}(x,y)$ is non-trivial since it cannot be obtained by the gradient of $f(x,y)$ directly. 
Hence, we maintain the stochastic estimator for the gradient operator of $f_{2L}(x,y)$ as follows: 
\begin{enumerate}
\item We first denote the gradient operator of $f_{2L}(x,y)$ by
\begin{align*}
F_{2L}(x,y) = \begin{bmatrix}
\nabla_x f_{2L}(x,y) \\ -\nabla_y f_{2L}(x,y)
\end{bmatrix},
\end{align*}
which also can be written as \citep{grimmer2020landscape}:
\begin{align*}
    F_{2L} (z)  = F(z^+) = 2L (z - z^+),
\end{align*}  
where $z=(x,y)$ and 
\begin{align*}
z^+ = (x^+, y^+) = \arg \min_{x' \in \BR^{d_x}} \max_{y' \in \BR^{d_y}} f(x',y') + L \Vert x' - x \Vert^2 - L \Vert y' - y \Vert^2.
\end{align*}
\item Then we estimate $z^+$ by $\hat z^+$, which is obtained by solving the minimax problem:
\begin{align} \label{eq:envelope-sub}
    \hat z^+ = (\hat x^+, \hat y^+) &\approx \arg \min_{x' \in \BR^{d_x}} \max_{y' \in \BR^{d_y}} g(x',y'; x,y),
\end{align}
where
$g(x',y'; x,y) \triangleq f(x',y') + L \Vert x' - x \Vert^2 - L \Vert y' - y \Vert^2$.
\item Finally, we construct $\hat F_{2L}(z) = 2L(z-  \hat z^+)$ as the estimator of  $F_{2L}(z)$.
\end{enumerate}

We would like to regard $\hat F_{2L}(\,\cdot\,)$ as the stochastic estimator of $F_{2L}(\,\cdot\,)$, then it looks reasonable to run RAIN on the saddle envelope $f_{2L}(\cdot)$ by using its approximate gradient operator $F_{2L}(\,\cdot\,)$. Since the function $g(x',y';x,y)$ is strongly-convex in $x'$ and strongly-concave in $y'$, it is desired to obtain $\hat z^+$ by Epoch-SEG efficiently. 
However, directly using Epoch-SEG on $g$ only leads to $\mathbb{E}\Vert \hat z^+ - z^+ \Vert^2$ be small, while the output $\hat z^+$ may be a biased estimator of $z^+$. Consequently, the constructed $\hat F_{2L}(z)$ would also be biased, violating the assumption of unbiased stochastic oracle in RAIN.
We address this issue by the following strategies:
\begin{enumerate}
\item We propose Epoch-SEG$^+$ (Algorithm \ref{alg:Epoch-SEG+}) by integrating the step of multilevel Monte-Carlo (MLMC)~\citep{asi2021stochastic} with Epoch-SEG, which makes the bias of $\hat z^+$ decrease exponentially.  
The formal theoretical result is described in Theorem~\ref{thm:Epoch-SEG+}.
\item We show that RAIN also works for stochastic oracles with low bias. 
Since the basic component of RAIN is SEG, it is sufficient to analyze the complexity of SEG with low biased stochastic oracle. 
The formal theoretical result is described in Theorem \ref{thm:SEG-bias}.
\end{enumerate}



\begin{algorithm*}[t]  
\caption{Epoch-SEG$^{+}$ $(f, z_0, \lambda, L, N, K, M)$} 
\begin{algorithmic}[1] \label{alg:Epoch-SEG+}
\STATE \textbf{for} $m = 0,1,\cdots,M-1$ \textbf{do} \\[0.1cm]
\STATE \quad draw $J \sim {\rm Geom}\left(1/2\right)$ \\[0.1cm]
\STATE \quad  get $z_{m,N}, z_{m,N+J-1},z_{m,N+J}$ by $\text{Epoch-SEG}(f,z_0, \lambda,L,N,J)$ \\[0.1cm]
\STATE \quad $\hat z_m \leftarrow z_{m,N} + 2^J (z_{m,N+J} - z_{m,N+J-1}) \mathbb{I}\,[J \le K]$ \\[0.1cm]
\STATE \textbf{return} $\hat z \leftarrow \frac{1}{M} \sum_{m=0}^{M-1} \hat z_m$.
\end{algorithmic}
\end{algorithm*}

\begin{thm} \label{thm:Epoch-SEG+}
Suppose the function $f(x,y)$ is $L$-smooth and $\lambda$-strongly-convex-$\lambda$-strongly-concave, and the stochastic first-order oracle  $F(x,y;\xi)$ is unbiased estimator of $F(x,y)$ and has variance bounded with $\sigma^2$, then {\rm Epoch-SEG}$^+$ (Algorithm \ref{alg:Epoch-SEG+}) holds that
\begin{align*}
\Vert \mathbb{E}[\hat z^*] - z^*\Vert^2 \le \frac{1}{2^{N+2K}}  \mathbb{E}\Vert z_0 - z^* \Vert^2 + \frac{8 \sigma^2}{2^K \lambda L}
\end{align*}
and
\begin{align*}
\mathbb{E}\Vert \hat z^* - \mathbb{E} [\hat z^*] \Vert^2 
&\le \frac{22}{2^N M}\mathbb{E}\Vert z_0 - z^* \Vert^2 + \frac{112 K \sigma^2}{M \lambda L}.
\end{align*}
Additionally, the total number of SFO calls is no more than
$16 \kappa MN + 64MKL/\lambda$ in expectation.
\end{thm}

\begin{thm}[SEG with biased oracle] \label{thm:SEG-bias}
Suppose the function $f(x,y)$ is $L$-smooth and $\lambda$-strongly-convex-$\lambda$-strongly-concave. 
We run {\rm SEG} (Algorithm~\ref{alg:SEG}) on $f(x,y)$ and denote
\begin{align*}
b_t &\triangleq  \Vert \mathbb{E}F(z_t;\xi_i) - F(z_t) \Vert, \qquad~\, b_{t+1/2} \triangleq \Vert \mathbb{E}F(z_{t+1/2};\xi_j) - F(z_{t+1/2}) \Vert, \\
\sigma^2_t & \triangleq {\BE}\Vert F(z_t;\xi_i) - F(z_t) \Vert^2, \qquad \sigma^2_{t+1/2}  \triangleq {\BE}\Vert F(z_{t+1/2};\xi_j) - F(z_{t+1/2}) \Vert^2.
\end{align*}
Then it holds that
\begin{align*}
\mathbb{E}\big[\lambda \Vert z_{t+1/2} - z^* \Vert^2\big] 
&\le {\frac{1}{\eta} \mathbb{E}\big[\Vert z_{t} - z^* \Vert^2 - \Vert z_{t+1} - z^* \Vert^2\big]} - \frac{1}{2 \eta} \mathbb{E}\Vert z_{t+1} - z_{t+1/2} \Vert^2 \\
&\quad + \underbrace{6 \eta(e_t^2 +e_{t+1/2}^2)}_{\text{mean square error}} + \underbrace{\frac{2 b_{t+1/2}^2}{\lambda}}_{\text{bias}}
\end{align*}
for any $0 < \eta < 1/(4L)$, where $e_t^2 \triangleq b_t^2 + \sigma_{t}^2$ and $e_{t+1/2}^2 \triangleq b_{t+1/2}^2 + \sigma_{t+1/2}^2$.
\end{thm}
\noindent The above theorem suggests that if it satisfies
\begin{align} \label{eq:require-RAIN+}
\sigma_t^2 \le \frac{\sigma^2}{2}, \quad \sigma_{t+1/2}^2 \le \frac{\sigma^2}{2}, \quad b_t^2 \le \frac{\sigma^2}{2}, 
\quad \text{and} \quad b_{t+1/2}^2 \le 2 \lambda \eta \sigma^2,
\end{align}
we are able to obtain the result of (\ref{eq:SEG-step}) in Lemma~\ref{lem:SEG}. 
Then we can follow the complexity analysis of RAIN
to establish the theoretical guarantee of applying RAIN on the saddle envelope.\footnote{We only need to replace every $L$ with $6L$ in Theorem \ref{thm:RAIN} and Theorem \ref{thm:EA-C} since $f_{2L}(x,y)$ is $6L$-smooth.} 
As a result, we can find an $\epsilon$-stationary point of $f_{2L}(x,y)$:
\begin{itemize}
\item For $(\tau,\alpha)$-intersection-dominant condition with $\tau \ge 2L$, it requires $\tilde \fO(\sigma^2 \epsilon^{-2} + L/\alpha)$ times evaluations of $\hat F_{2L}(\,\cdot\,)$.
\item For $-\rho$-comonotone condition with $\rho \le 1/(2L)$, it requires 
$\tilde \fO(\sigma^2 \epsilon^{-2} + L/\epsilon)$ times evaluations of $\hat F_{2L}(\,\cdot\,)$.
\end{itemize}

Based on the above ideas, we proposed RAIN$^{++}$ in Algorithm \ref{alg:RAIN++:expand} as an extension of RAIN. 
The iterations of RAIN$^{++}$ consider the regularized function
\begin{align} \label{eq:env-s}
f_{2L}^{(s)}(x,y) \triangleq \begin{cases}
f_{2L} (x,y), & s=0, \\
\displaystyle{f_{2L}^{(s-1)}(x,y)+\frac{\lambda_s}{2}\|x-x_s\|^2-\frac{\lambda_s}{2}\|y-y_s\|^2}, & s\geq 1,
\end{cases}
\end{align}
Compared with the iterations of RAIN applying SEG to solve the sub-problem
\begin{align*}
\min_{x\in\BR^{d_x}}\max_{y\in\BR^{d_y}}f^{(s)}(x,y),    
\end{align*}
the iterations of RAIN$^{++}$ apply SEG${}^{++}$ to solve the sub-problem
\begin{align*}
\min_{x\in\BR^{d_x}}\max_{y\in\BR^{d_y}}f_{2L}^{(s)}(x,y).
\end{align*}
For SEG${}^{++}$, we apply Epoch-SEG${}^+$ (Epoch-SEG with the debias step of MLMC) to achieve the low-biased gradient estimation of the saddle envelope $f_{2L}^{(s)}(x,y)$.
Note that the sub-problems solved by Epoch-SEG${}^+$ (Line 5 and 9 of Algorithm \ref{alg:SEG++}) 
are well-conditioned because they are $3L$-smooth and $L$-strongly-convex-$L$-strongly-concave.
Therefore the SFO complexity for Epoch-SEG$^+$ only require the complexity of $\tilde \fO(1)$ to obtain a sufficiently accurate solution by carefully choosing the initialization and parameters. 
As a result, the total SFO complexity of RAIN$^{++}$ would be $\tilde \fO(1)$ times the one of RAIN.

\begin{algorithm*}[t]  
\caption{RAIN$^{++}$ $(f,z_0,\lambda, L ,\{ N_s\}_{s=0}^{S-1}, \{K_s \}_{s=0}^{S-1})$} 
\begin{algorithmic}[1] \label{alg:RAIN++:expand}
\STATE $\lambda_0 = \lambda \gamma, \quad S = \lfloor \log_{(1+\gamma)}(6 L/\lambda) \rfloor, \quad w_0 \leftarrow z_0$ \\[0.1cm]
\STATE \textbf{for} $s = 0,1,\cdots, S-1$ \\[0.1cm]
\STATE \quad $z_{s,0} \leftarrow z_s, \quad w_{s,0} \leftarrow w_s$ \\[0.1cm]
\STATE \quad \textbf{for} $k = 0,1,\cdots,N_s-1$ \textbf{do} \\[0.1cm]
\STATE \quad \quad $(z_{s,k+1},w_{s,k+1}) \leftarrow \text{SEG}^{++} \big(f,z_{s,k},w_{s,k}, \frac{1}{48 L}, \frac{96 L}{\lambda_s}, s, \{ z_i\}_{i=0}^s, \gamma, \lambda\big)$  \\[0.1cm]
\STATE \quad \textbf{for} $k = N_s,N_s+1,\cdots,N_s+K_s+1$ \textbf{do} \\[0.1cm]
\STATE \quad \quad $(z_{s,k+1},w_{s,k+1}) \leftarrow \text{SEG}^{++} \big(f,z_{s,k},w_{s,k}, \frac{1}{2^{k-N_s+3}\times 12 L}, \frac{2^{k- N_s +5} \times 12 L}{\lambda_s}, s, \{ z_i\}_{i=0}^s, \gamma, \lambda\big)$ 
\STATE \quad \textbf{end for} \\[0.1cm]
\STATE \quad $(z_{s+1},w_{s+1}) \leftarrow (z_{N_s+K_s},w_{N_s+K_s})$ \\[0.1cm]
\STATE \quad $\lambda_{s+1} \leftarrow (1+ \gamma) \lambda_s$ \\[0.1cm]
\STATE \textbf{end for} \\[0.1cm]
\STATE \textbf{return} $z_S$ 
\end{algorithmic}
\end{algorithm*}

Below, we present the theoretical result under the intersection-dominant condition.

\begin{thm}[RAIN${}^{++}$, ID] \label{thm:RAIN++-ID}
Suppose the function $f(x,y)$ is $L$-smooth and satisfies the $(\tau,\alpha)$-intersection-dominant condition with $\tau \ge 2L$, and the stochastic first-order oracle $F(x,y;\xi)$ is unbiased and has variance bounded by $\sigma^2$, then running {\rm RAIN}${}^{++}$ with $\gamma = 1$, $\lambda$ as defined in (\ref{rela:alpha-mu}) and
\begin{align*} 
N_s = 
\begin{cases}
\left \lceil \log_2 \left(\dfrac{512 \lambda^2 S^2 D^2}{\epsilon^2} \right ) \right \rceil, & s = 0,  \\[0.1cm]
3, & s \ge 1,
\end{cases}
\qquad\text{and}\qquad
K_s &= \left \lceil \log_2 \left(\frac{1024 \lambda_s S^2 \sigma^2}{3 L \epsilon^2} \right)\right \rceil
\end{align*}
holds that $\mathbb{E}\Vert F_{2L}(z_S) \Vert\le \epsilon$, where $\lambda$ is defined in (\ref{rela:alpha-mu}). Additionally, the total number of SFO calls is no more than $\tilde\fO(\sigma^2 \epsilon^{-2} + L / \alpha)$ in expectation.
\end{thm}
We can also use the regularization trick to obtain the theoretical guarantee in the negative comonotone setting. The idea is applying RAIN$^{++}$ on
\begin{align*}
    g_{2L}(x,y) \triangleq f_{2L}(x,y) + \frac{\lambda}{2} \Vert x - x_0 \Vert^2 - \frac{\lambda}{2} \Vert y -  y_0 \Vert^2
\end{align*}
for some $\lambda = \Theta(\epsilon/D)$. Note that the function $g_{2L}(x,y)$ is $\lambda$-strongly-convex-$\lambda$-strongly-concave and the algorithm solves the minimax sub-problems with objective functions
\begin{align} \label{eq:env-s-nc}
f_{2L}^{(s)}(x,y) \triangleq \begin{cases}
\displaystyle{f_{2L} (x,y) + \frac{\lambda}{2} \Vert x - x_0 \Vert - \frac{\lambda}{2} \Vert y - y_0 \Vert^2}, & s=0, \\
\displaystyle{f_{2L}^{(s-1)}(x,y)+\frac{\lambda_s}{2}\|x-x_s\|^2-\frac{\lambda_s}{2}\|y-y_s\|^2}, & s\geq 1,
\end{cases}
\end{align}
which is similar to what we have done in the general convex-concave case. 
We formally present the result in the following theorem.
\begin{thm}[RAIN${}^{++}$, NC] \label{thm:RAIN++-NC}
Suppose the function $f(x,y)$ is $L$-smooth and its gradient operator is $-\rho$-comonotone with $\rho \le 1/(2L)$, and the stochastic first-order oracle $F(x,y;\xi)$ is unbiased and has variance bounded by $\sigma^2$. If we run {\rm RAIN}$^{++}$ (Algorithm \ref{alg:RAIN++:expand}) with $\gamma = 1$ and $\lambda = \min\{ \epsilon/D, 6L \}$ then it holds that $\mathbb{E}\Vert F_{2L}(z_S) \Vert\le 3 \epsilon$.
Additionally, the total number of SFO calls is no more than $\tilde \fO(\sigma^2 \epsilon^{-2} + L { D} \epsilon^{-1})$.
\end{thm}

Corollary~\ref{thm:RAIN++-ID} and Corollary~\ref{thm:RAIN++-NC} indicate RAIN$^{++}$ can find an $\epsilon$-stationary point of the envelope $f_{2L}(x,y)$ in corresponding settings, which easily leads to a nearly-stationary point of $f(x,y)$ by the following proposition.

\begin{prop}\label{prop:final}
Suppose the function $f(x,y)$ is $L$-smooth and the point $\hat z = (\hat x, \hat y)$ is an $\epsilon$-stationary point of the function $f_{2L}(x,y)$, then we can find a $2\eps$-stationary point of $f(x,y)$ within $\fO(\log(\epsilon^{-1})+ \sigma^2 \epsilon^{-2})$ SFO complexity in expectation.
\end{prop}


\begin{algorithm*}[t]  
\caption{SEG${}^{++}$ $(f, z_0, w_0, \eta, T, s, \{\tilde z_i\}_{i=0}^s,\gamma, \lambda)$} 
\begin{algorithmic}[1] \label{alg:SEG++}
\STATE let $L$  be the smoothness coefficient of $f(x,y)$  \\[0.05cm]
\STATE set $N,K,M$ according to equation (\ref{value:NMK}) and $\hat z_{-1/2}^+ \leftarrow w_0$  \\[0.05cm]
\STATE \textbf{for} $t = 0,1,\cdots,T-1$ \textbf{do} \\[0.05cm]
\STATE \quad $g_t(x,y) \triangleq f(x,y) + L \Vert x - x_t \Vert^2 - L \Vert y - y_t \Vert^2$ \\[0.05cm] \label{line:gt}
\STATE \quad $\hat z_t^+ \leftarrow \text{Epoch-SEG}^{+} (g_t,\hat z_{t-1/2}^+,L, 3L,N,K,M)$ \\[0.05cm]
\STATE \quad $\hat F_{2L}^{(s)}(z_t) \leftarrow \begin{cases}
2L(z_t - \hat z_t^+) + \lambda \gamma \sum_{i=1}^s (1+\gamma)^i (z_t - \tilde z_i), & \text{ID case} \\
2L(z_t - \hat z_t^+) + \lambda \gamma \sum_{i=0}^s (1+\gamma)^i (z_t - \tilde z_i) , & \text{NC case}
\end{cases}$ \\[0.05cm]
\STATE \quad $z_{t+1/2} \leftarrow z_t - \eta \hat F_{2L}^{(s)}(z_{t})$ \\[0.05cm]
\STATE \quad $g_{t+1/2}(x,y) \triangleq f(x,y) + L \Vert x - x_{t+1/2}\Vert^2 - L \Vert y - y_{t+1/2} \Vert^2$ \\[0.05cm] \label{line:gt12}
\STATE \quad $\hat z_{t+1/2}^+ \leftarrow \text{Epoch-SEG}^{+} (g_{t+1/2},\hat z_{t}^+,L, 3L,N,K,M)$ \\[0.05cm]
\STATE \quad $\hat F_{2L}^{(s)}(z_{t+1/2}) \leftarrow 
\begin{cases}
2L(z_{t+1/2} - \hat z_{t+1/2}^+) + \lambda \gamma  \sum_{i=1}^s (1+\gamma)^i (z_{t+1/2} - \tilde z_i), & \text{ID case} \\
2L(z_{t+1/2} - \hat z_{t+1/2}^+) + \lambda \gamma  \sum_{i=0}^s (1+\gamma)^i (z_{t+1/2} - \tilde z_i), & \text{NC case}
\end{cases}$\\[0.05cm]
\STATE \quad $z_{t+1} \leftarrow z_t - \eta \hat F_{2L}^{(s)}(z_{t+1/2})$ \\[0.05cm]
\STATE \textbf{end for} \\[0.05cm]
\STATE Draw $J \sim {\rm Unif}([T])$ \\[0.05cm] 
\STATE \textbf{return} $(z_{J+1/2}, \hat z_{J+1/2})$ 
\end{algorithmic}
\end{algorithm*}

\section{The Lower Complexity Bounds}\label{sec:lower-bound}

We first provide the lower complexity bounds for finding near-stationary points of the stochastic convex-concave minimax problem.
Combining the ideas of \citet{luo2021near,foster2019complexity} and \citet{yoon2021accelerated}, we obtain Theorem~\ref{thm:lower-C} and Theorem~\ref{thm:lower-SC} for general convex-concave case and strongly-convex-strongly-concave case respectively, which nearly match the upper bounds shown in Theorem \ref{thm:RAIN} and Theorem~\ref{thm:EA-C}. 
Hence, the proposed RAIN (Algorithm~\ref{alg:RAIN}) is near-optimal. 


\begin{thm} \label{thm:lower-SC}
For any stochastic algorithm $\fA$ based on stochastic first-order oracle (SFO) under Assumption~\ref{asm:bounded-init} and parameters $L \ge 2, \lambda \le 1$ and $\epsilon \le 0.01 \lambda$, there exist an $L$-smooth and $\lambda$-strongly-convex-$\lambda$-strongly-concave function $f(x,y)$ that $\fA$ needs at least 
\begin{align*}
    \Omega \left( \sigma^2 \epsilon^{-2} \log(L \epsilon^{-1}) + \kappa \log(\lambda \epsilon^{-1})\right)
\end{align*}
SFO calls to find an $\epsilon$-stationary point of $f(x,y)$.
\end{thm}

\begin{thm} \label{thm:lower-C}
For any stochastic algorithm $\fA$ based on stochastic first-order oracle (SFO) under Assumption~\ref{asm:bounded-init} and parameters of $L\ge 2$ and $\epsilon \le 0.01$, there exists an $L$-smooth and convex-concave function $f(x,y)$ such that $\fA$ needs at least 
\begin{align*}
    \Omega \left( \sigma^2 \epsilon^{-2} \log(L \epsilon^{-1}) + L { D}\epsilon^{-1} \right)
\end{align*}
SFO calls to find an $\epsilon$-stationary point of $f(x,y)$.
\end{thm}

Since the gradient operator of any convex-concave function is negative comonotone, the lower bound shown in Theorem~\ref{thm:lower-SC} is also valid for the problem under the negative comonotone condition.

\begin{cor} \label{cor:lower-NC}
For any stochastic algorithm $\fA$ based on stochastic first-order oracle (SFO) under Assumption~\ref{asm:bounded-init} and parameters of $L \ge 2$ and $\epsilon \le 0.01$, there exist an $L$-smooth function $f(x,y)$ whose gradient operator is negative comonotone such that $\fA$ needs at least
\begin{align*}
    \Omega \left( \sigma^2 \epsilon^{-2} \log(L \epsilon^{-1}) + L { D} \epsilon^{-1} \right)
\end{align*}
SFO calls to find an $\epsilon$-stationary point of $f(x,y)$.
\end{cor}

\noindent Similarly, Theorem~\ref{thm:lower-C} leads to the lower bound for intersection-dominant condition.

\begin{cor}\label{cor:lower-ID}
For any stochastic algorithm $\fA$ based on stochastic first-order oracle (SFO) under Assumption~\ref{asm:bounded-init} and parameters of $L\ge 2$, $\alpha \le 1$ and $\epsilon \le 0.01 \alpha$, there exist an $L$-smooth function $f(x,y)$ satisfying the $(\tau,\alpha)$-intersection-dominant condition with $\tau \ge 2L$ such that $\fA$ needs at least
\begin{align*}
    \Omega \left( \sigma^2 \epsilon^{-2} \log(L \epsilon^{-1}) + L \alpha^{-1} \log(\alpha \epsilon^{-1})\right)
\end{align*}
SFO calls to find an $\epsilon$-stationary point of $f(x,y)$.
\end{cor}

\noindent Corollary \ref{cor:lower-NC} and Corollary \ref{cor:lower-ID} suggest that the proposed RAIN$^{++}$ (Algorithm~\ref{alg:RAIN++:expand}) is also near-optimal in negative comonotone and intersection-dominant conditions respectively.

\begin{figure}[t]
    \centering
    \begin{tabular}{ccc}
    \includegraphics[scale=0.31]{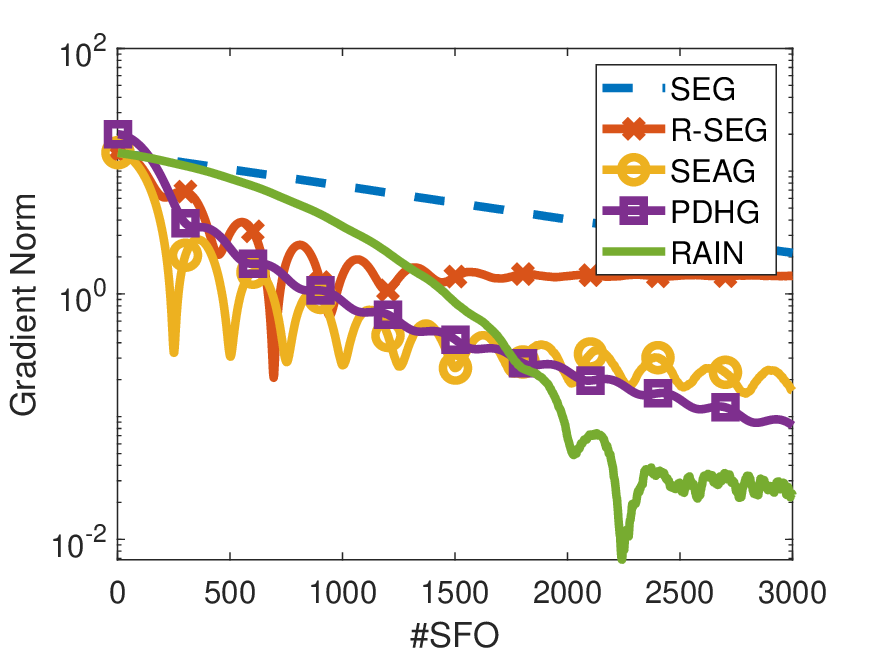} &
    \includegraphics[scale=0.31]{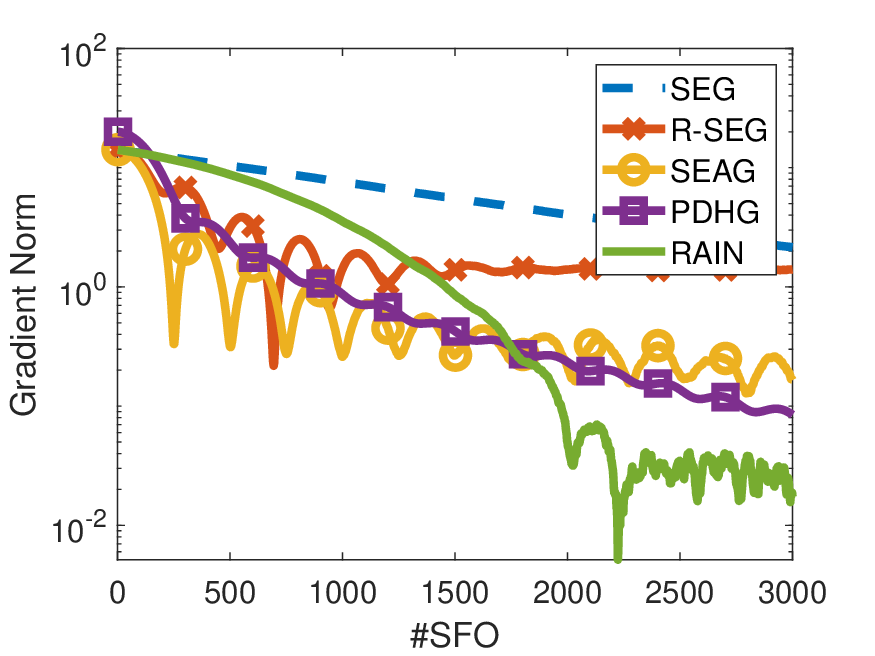}  & \includegraphics[scale=0.31]{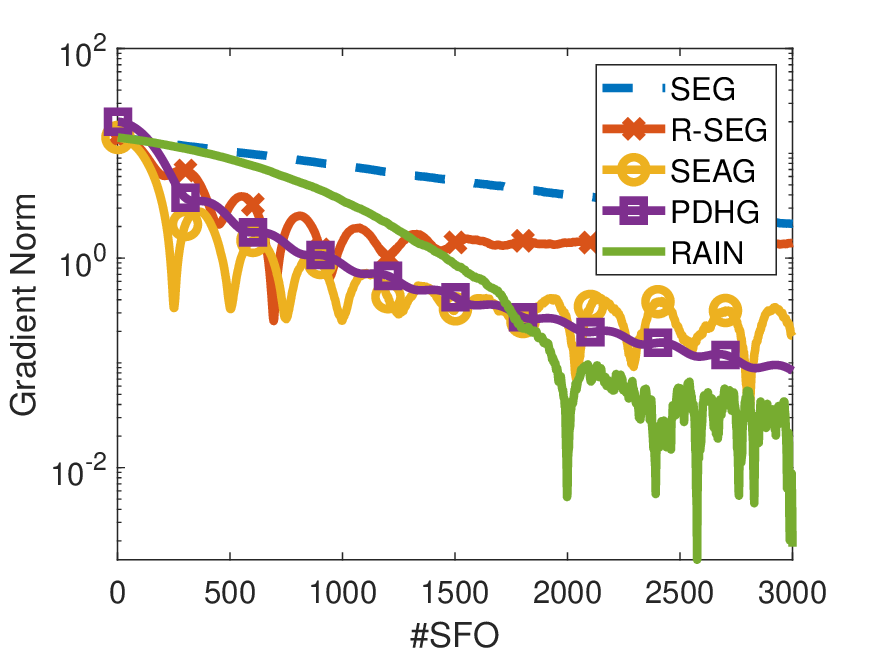} \\
    (a) $\sigma = 0.001$ & (b) $\sigma = 0.002$ & (c) $\sigma = 0.005$ \\
    \end{tabular}
    \caption{The results of the number of SFO calls against gradient norm on problem (\ref{func-xy}).}
    \label{fig:Exp-xy} \vskip0.3cm
\end{figure}

\begin{figure}[t]
    \centering
	\begin{tabular}{ccc}
		\includegraphics[scale=0.31]{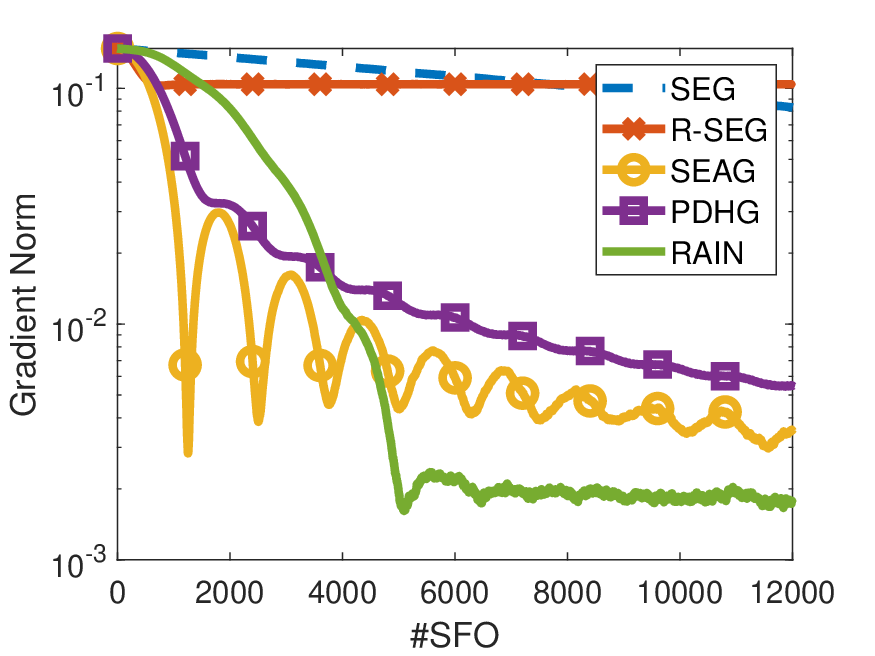} 
		&\includegraphics[scale=0.31]{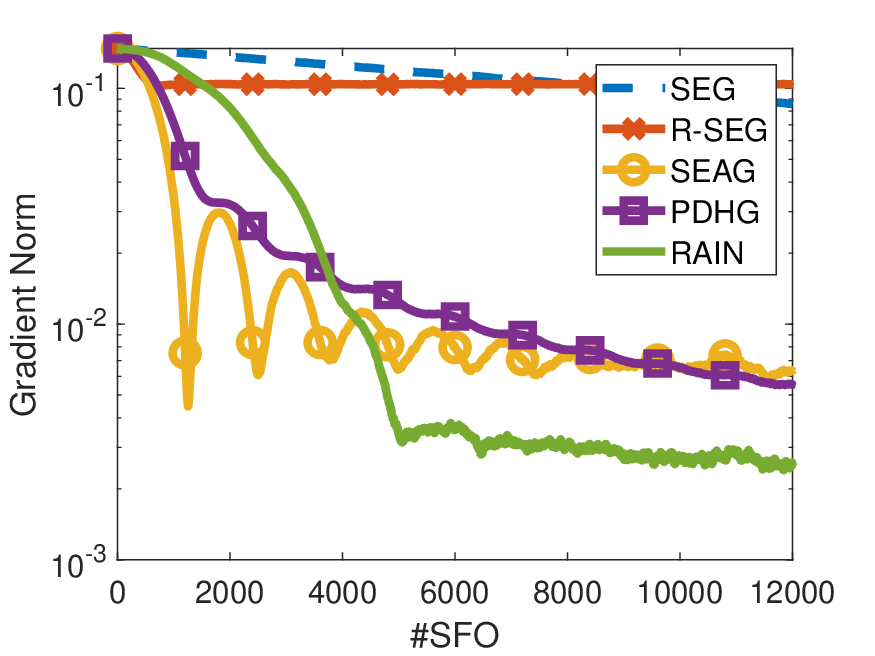} 
		&\includegraphics[scale=0.31]{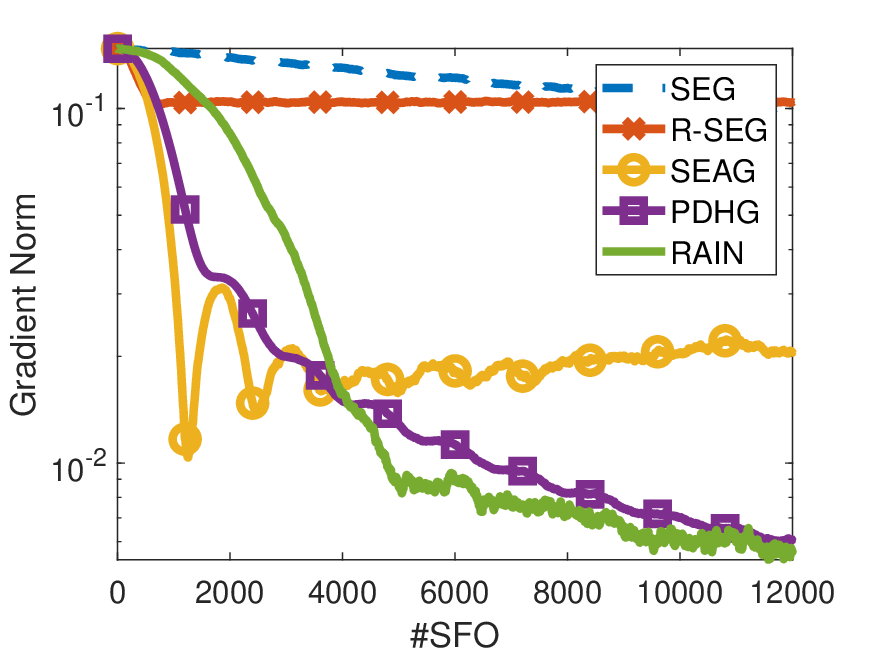} \\
        (a) $\sigma = 0.001$ & (b) $\sigma = 0.002$  & (c) $\sigma = 0.005$ \\
	\end{tabular} 
    \caption{The results of the number of SFO calls against gradient norm on problem (\ref{func-delta}). 
    SEAG diverges in (c), which does not contradict its convergence guarantee as 
    the condition $\sigma_k^2 \le \epsilon / (k+1)$ in Theorem 6.1~\citep{lee2021fast} is unsatisfied.
    }
    \label{fig:Exp-delta}
\end{figure}

\section{Numerical Experiments}

In this section, we compare our algorithms with baselines on both convex-concave and nonconvex-nonconcave minimax problems.

\subsection{The Convex-Concave Case}
We conduct numerical experiments by comparing our RAIN with the following baselines:
\begin{itemize}
    \item SEG: the ordinary stochastic extragradient;
    \item R-SEG: the stochastic extragradient with regularization trick \citep{nesterov2012make};
    \item SEAG: the stochastic extra anchored gradient  \citep[Algorithm 3]{lee2021fast};
    \item {PDHG: the primal-dual hybrid gradient \citep[Algorithm 1]{zhao2022accelerated}}. 
\end{itemize}
The experiments consider two minimax problems as follows.
\begin{itemize}
\item The first one is the bilinear minimax problem:
\begin{align}\label{func-xy}
    \min_{x\in\BR^d} \max_{y\in\BR^d} f_{\rm b}(x,y) \triangleq x^\top y,
\end{align}
which reveals some important issues in minimax optimization. 
For example, the duality gap is not well defined except at its unique saddle point $(0,0)$; the classical method stochastic gradient descent ascent diverges and other first-order algorithms also converge slowly due to the cycling behaviors \citep{pethick2022escaping}. 
\item The second one is the hard case of convex-concave minimax problem
\begin{align} \label{func-delta}
    \min_{x\in\BR^d} \max_{y\in\BR^d} f_{\delta, \nu}(x,y) \triangleq (1 - \delta) g_{\nu}(x) + \delta x^\top y - (1 - \delta) g_{\nu} (y), 
\end{align}
where 
\begin{align*}
    g_{\nu}(u_i) = 
    \begin{cases}
    \nu \vert u_i \vert  -  \frac{1}{2} \nu^2 , & \vert u_i \vert \ge \nu, \\
    \frac{1}{2} u_i^2 , & \vert u_i \vert < \nu,
    \end{cases}
\end{align*}
and we set $\nu = 5 \times 10^{-5}$ and $\delta = 10^{-2}$ by following~\citet{yoon2021accelerated}'s setting.
\end{itemize}
We use the stochastic first-order oracles
\begin{align*}
 F_b(x,y;\xi) = F_b(x,y) + \xi 
\qquad \text{and} \qquad
 F_{\delta,\nu}(x,y;\xi) = F_{\delta,\nu}(x,y) + \xi 
\end{align*}
for problems (\ref{func-xy}) and (\ref{func-delta}) respectively, where $F_b(x,y)$ is the gradient operator of $f_b(x,y)$, $F_{\delta,\nu}(x,y)$ is the gradient operator of $f_{\delta,\nu}(x,y)$ and $\xi\sim\fN(0,\sigma^2I_{2d})$.
We set $d = 1000$ for problem (\ref{func-xy}) and $d=100$ for problem (\ref{func-delta}). 
We provide the detailed implementation for the algorithms in Appendix~\ref{apx:exp} and the source code is available\footnote{\url{https://github.com/TrueNobility303/RAIN}}.
We present the experimental results under different levels of noise in Figure~\ref{fig:Exp-xy} and Figure~\ref{fig:Exp-delta}, which shows the proposed RAIN obviously performs better than baseline methods.

\subsection{The Nonconvex-Nonconcave Case}

We also consider the 
nonconvex-nonconcave problem 
\begin{align} \label{func-rho}
    \min_{x\in\BR}\max_{y\in\BR} f_{\rho,L}(x,y) \triangleq \frac{\rho L^2}{2} x^2 + L \sqrt{1 - \rho^2 L^2} xy  -\frac{\rho L^2}{2} y^2,
\end{align} 
where $L,\rho>0$ such that $\rho L<1$.
According to the verification of \citet{lee2021fast}, the gradient operator of function $f_{\rho,L}(x,y)$ is $L$-Lipschitz continuous and $\rho$-comonotone. 
We use the stochastic first-order oracles 
\begin{align*}
F_{\rho,L}(x,y;\xi) = f_{\rho,L}(x,y) + \xi
\end{align*}
in our experiments, where $\xi \sim \fN(0,\sigma^2 I_2)$.

We compare our proposed RAIN${}^{++}$ with the following baselines:
\begin{itemize}
    \item SEG${}^+$: the extension of SEG under the negative comonotonicity assumption ~\citep[Equation (EG${}_{p}^+$)]{diakonikolas2021efficient};
    \item SFEG: the stochastic fast extra gradient method \citep[Algorithm 1]{lee2021fast}, where we replace the exact gradient operator with the stochastic gradient operator. 
\end{itemize}
We test the algorithms on Problem (\ref{func-rho}) with $L=1$ and $\rho =-{1}/{8\sqrt{2}}$ or $ -{1}/{3}$, and the experimental results are shown in  
Figure \ref{fig:Exp-rho} and \ref{fig:Exp-rho-div}, respectively.
The implementation details of the algorithm are deferred to Appendix \ref{apx:exp}. We can observe that our proposed RAIN${}^{++}$ performs better than baselines.
We remark that the convergence analysis of SFEG by \citet{lee2021fast} requires either $\rho = 0$ or $\sigma=0$, while it also works on our stochastic nonconvex-nonconcave problem with $\rho>0$ and $\sigma>0$ in practice.
Note that Figure \ref{fig:Exp-rho-div} shows that SEG${}^+$ diverges on Problem (\ref{func-rho}) with $\rho=-1/3$, which is also observed by the empirical results of \citet{lee2021fast}.
The reason is the convergence guarantee of SEG${}^+$ \citep[Theorem 4.5]{diakonikolas2021efficient}) requires the condition $\rho \in [-{1}/{(4 \sqrt{2} L)}, 0)$, which is not satisfied in the setting of $\rho=-1/3$.

\begin{figure}[t]
    \centering
    \begin{tabular}{ccc}    
    \includegraphics[scale=0.31]{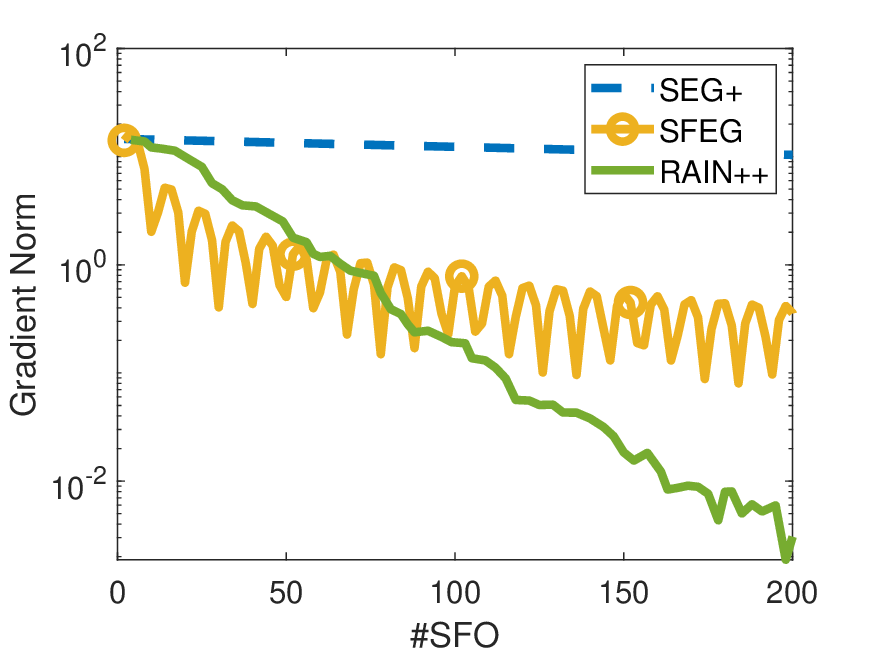} &\includegraphics[scale=0.31]{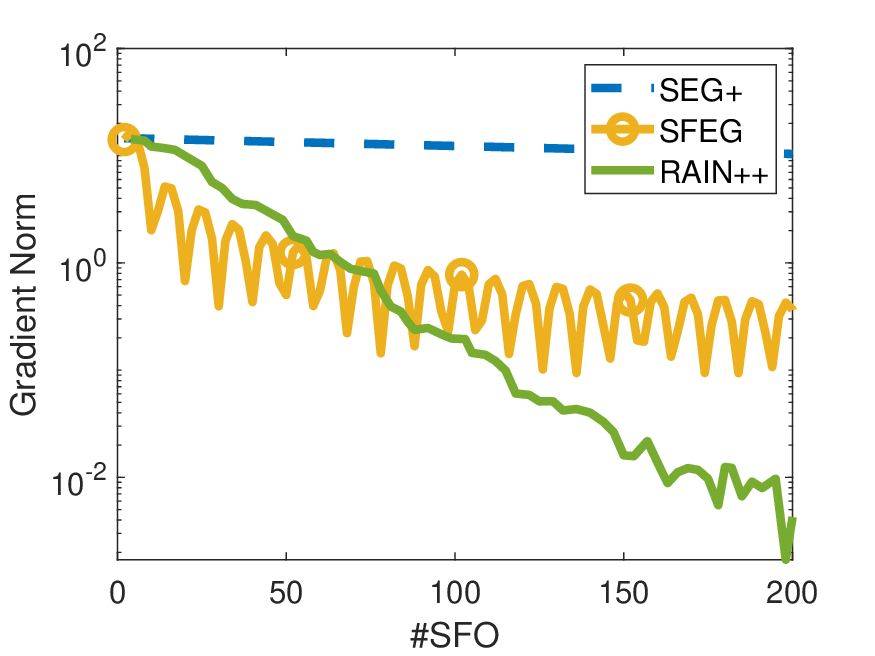}  & \includegraphics[scale=0.31]{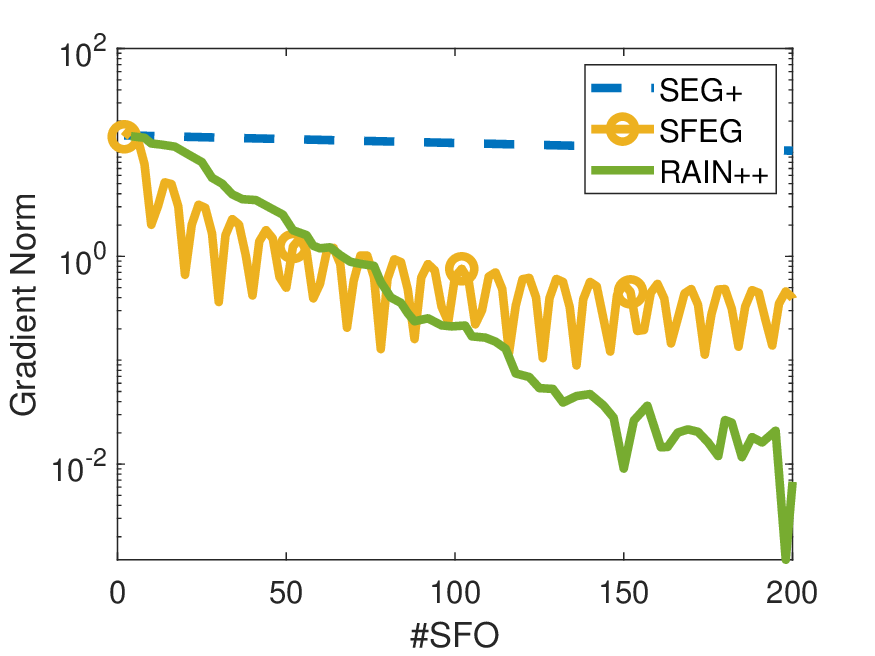} \\
    (a) $\sigma = 0.001$ & (b) $\sigma = 0.002$ & (c) $\sigma = 0.005$ \\
    \end{tabular}
    \caption{The results of the number of SFO calls against gradient norm on problem (\ref{func-rho}) with $L=1$ and $\rho = - {1}/{(8 \sqrt{2})}$.}
    \label{fig:Exp-rho}
\end{figure}

\begin{figure}[htbp]
    \centering
    \begin{tabular}{ccc}
    \includegraphics[scale=0.31]{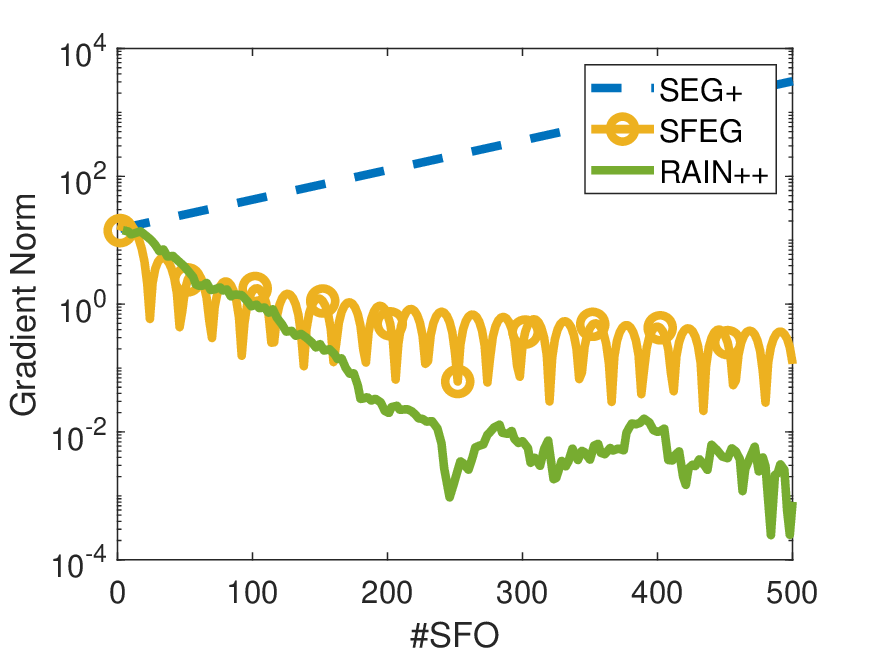} &\includegraphics[scale=0.31]{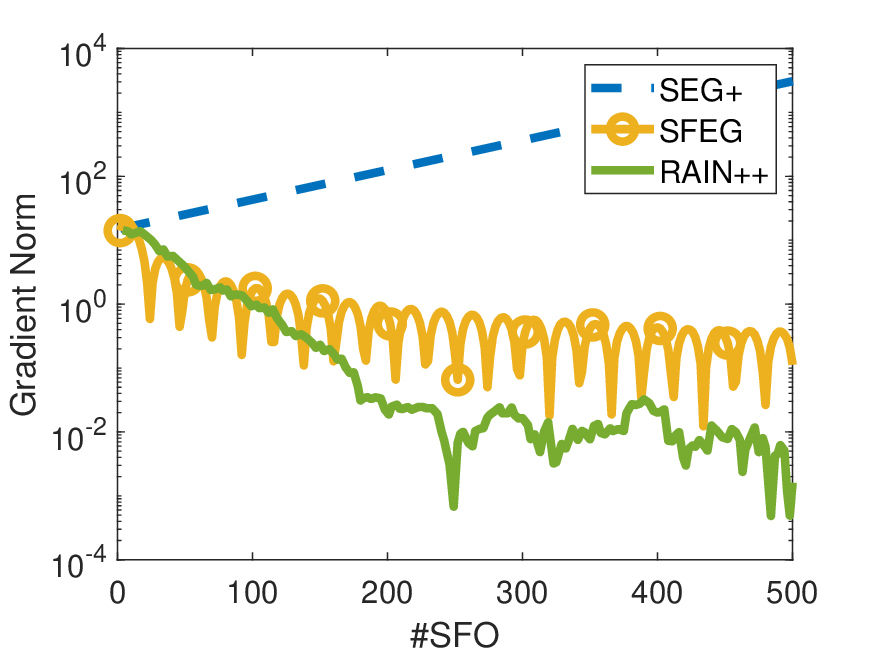}  & \includegraphics[scale=0.31]{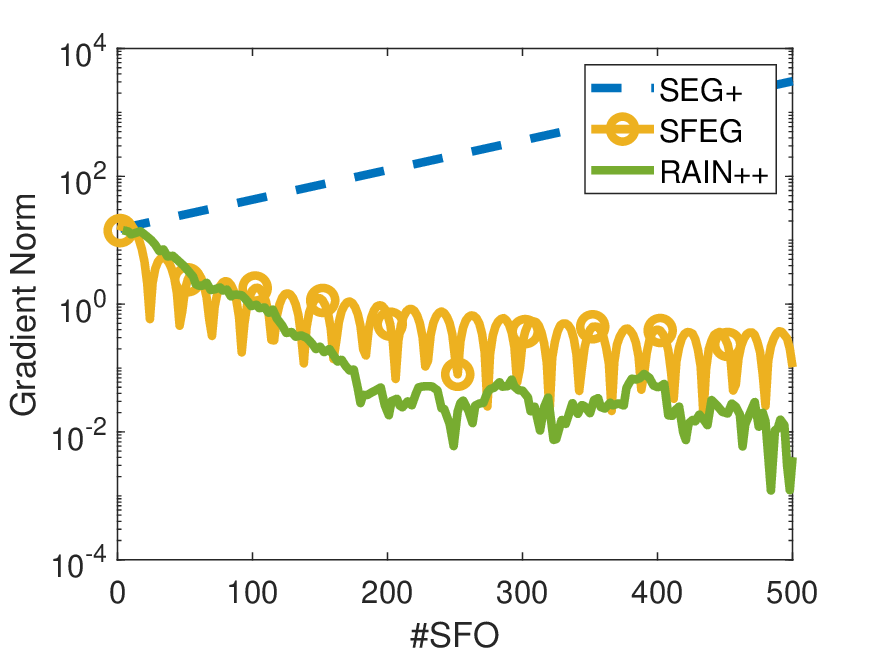} \\
    (a) $\sigma = 0.001$ & (b) $\sigma = 0.002$ & (c) $\sigma = 0.005$ \\
    \end{tabular}
    \caption{The results of the number of SFO calls against gradient norm on problem (\ref{func-rho}) with $L=1$ and $\rho = - {1}/3$.
    \label{fig:Exp-rho-div}}
\end{figure}



\section{Conclusion}
In this work, we propose the Recursive Anchor IteratioN (RAIN) algorithm for stochastic minimax optimization. The theoretical analysis has shown that the framework of RAIN with appropriate sub-problem solvers could achieve the near-optimal SFO complexity for finding nearly stationary points in convex-concave and strongly-convex-strongly-concave minimax optimization problems. 
We also extend the idea of RAIN to solve two specific nonconvex-nonconcave minimax problems and the proposed method RAIN$^{++}$ also achieves the near-optimal SFO complexity in these settings.

\section*{Acknowledgments}
This work was done while Lesi Chen was at Fudan University.
This work is supported by National Natural Science Foundation of China (No. 62206058),
Shanghai Sailing Program (22YF1402900), Shanghai Basic Research Program (23JC1401000), 
and the Major Key Project of Pengcheng Laboratory (No. PCL2024A06)


\newpage
\appendix
\section{The Proofs in Section~\ref{sec:RAIN}}

We first provide the non-expansiveness lemma, then give detailed proofs for results in Section~\ref{sec:RAIN}.

\begin{lem}[non-expansiveness]  \label{lem:non-ex}
Suppose the operator $F(\,\cdot\,)$ is monotone. 
We define $G(z) \triangleq F(z) + \lambda (z - z_0) $ for some $\lambda >0$ and $z_0 \in \BR^d$,  then it holds that
\begin{align*}
\Vert w^* - z_0 \Vert \le \Vert z^* - z_0 \Vert
\qquad\text{and}\qquad
\Vert w^* - z^* \Vert \le \Vert z^* - z_0 \Vert,
\end{align*}
where $z^*$ is the solution to $F(z) = 0$ and $w^*$ is the solution to $G(w) = 0$. 
\begin{proof}
The monotonicity of $F$ means $G$ is $\lambda$-strongly monotone. Therefore, the operator $G$ must has an unique solution $w^*$. Then we have
\begin{align*}
    & \lambda \Vert w^* - z^* \Vert^2 \\
    \le& G(z^*)^\top (z^* - w^*) \\
    =& (F(z^*) + \lambda (z^* - z_0) )^\top (z^* - w^*)  \\
    =& \lambda (z^* - z_0)^\top (z^* - w^*) \\
    =& \frac{\lambda}{2} \Vert w^* - z^* \Vert^2 + \frac{\lambda}{2} \Vert z^* - z_0 \Vert^2 - \frac{\lambda}{2} \Vert w^* - z_0 \Vert^2,
\end{align*}
which implies the result of this lemma.
\end{proof}
\end{lem}

\subsection{The Proof of Lemma \ref{lem:RAL}}
\begin{proof}
Let $F^{(s)}(\cdot)$ be the gradient operator of $f^{(s)}(x,y)$. It is clear that $F^{(s)}(\cdot)$ is equivalent to the following definition:
\begin{align*}
    F^{(s)}(z) \triangleq F(z) + \lambda \sum_{i=1}^s 2^i (z - z_i), \quad s = 0,\cdots, S-1.
\end{align*}
Based on the setting of $S = \lfloor \log_2 (L/\lambda) \rfloor$, we have
\begin{align*}
    \lambda 2^S \le L \le \lambda 2^{S+1},
\end{align*}
then it holds that
\begin{align*}
\lambda + \lambda \sum_{i=1}^s 2^i > 2^s \lambda
\end{align*}
and
\begin{align*}
\quad L + \lambda \sum_{i=1}^{S-1} 2^i \le 2L.
\end{align*}
Hence, we know that each $F^{(s)}$ is at least $2^s \lambda$-strongly monotone and $2L$-smooth. Further, we have
\begin{align*}
    \Vert F(z_S) \Vert & \le \Vert F^{(S-1)}(z_S) \Vert + \lambda \sum_{i=1}^{S-1} 2^i \Vert z_S - z_i \Vert \\
    &\le  \Vert F^{(S-1)}(z_S) \Vert + \lambda \sum_{i=1}^{S-1} 2^i \Vert z_S - z_{S-1}^* \Vert + \lambda \sum_{i=1}^{S-1} 2^i \Vert z_{S-1}^* - z_i \Vert \\
    &\le 2 \Vert F^{(S-1)}(z_S) \Vert + \lambda \sum_{i=1}^{S-1} 2^i \Vert z_{S-1}^* - z_i \Vert \\
    &\le  2 \Vert F^{(S-1)}(z_S) \Vert + \lambda \sum_{i=1}^{S-1} 2^i \Vert z_{i-1}^* - z_i \Vert + \lambda \sum_{i=1}^{S-1} 2^i \Vert z_{S-1}^* - z_{i-1}^* \Vert \\
    &\le 2 \Vert F^{(S-1)}(z_S) \Vert +  \lambda \sum_{i=1}^{S-1} 2^i \Vert z_{i-1}^* - z_i \Vert + \lambda \sum_{i=1}^{S-1} 2^i \sum_{j=i}^{S-1} \Vert z_{j}^* - z_{j-1}^* \Vert \\
    &= 2 \Vert F^{(S-1)}(z_S) \Vert +  \lambda \sum_{i=1}^{S-1} 2^i \Vert z_{i-1}^* - z_i \Vert + \lambda \sum_{j=1}^{S-1} \Vert z_{j}^* - z_{j-1}^* \Vert \sum_{i=1}^{j} 2^i  \\
    &\le  2 \Vert F^{(S-1)}(z_S) \Vert +  \lambda \sum_{i=1}^{S-1} 2^i \Vert z_{i-1}^* - z_i \Vert + \lambda \sum_{j=1}^{S-1} 2^{j+1} \Vert z_j^* - z_{j-1}^* \Vert \\
    &= 2 \Vert F^{(S-1)}(z_S) \Vert +  \lambda \sum_{i=1}^{S-1} 2^i \Vert z_{i-1}^* - z_i \Vert + \lambda \sum_{i=1}^{S-1} 2^{i+1} \Vert z_i^* - z_{i-1}^* \Vert \\
    &\le 4L \Vert z_S - z_{S-1}^* \Vert +  3\lambda \sum_{i=1}^{S-1} 2^i \Vert z_{i-1}^* - z_i \Vert \\
    &\le 16 \lambda \sum_{i=1}^{S} 2^{i-1} \Vert z_{i-1}^* - z_i \Vert.
\end{align*}
Above, the third line follows from $F^{(S-1)}$ is at least $\big(\lambda \sum_{i=1}^{S-1} 2^i\big)$-strongly monotone; the second last line relies on the non-expansiveness after anchoring shown in Lemma~\ref{lem:non-ex} such that $\Vert z_i^* - z_{i-1}^* \Vert \le \Vert z_{i-1}^* - z_i \Vert$; in the last line we use $F^{(S-1)}$ is at most $2L$-Lipschitz and $L \le \lambda 2^{S+1}$; the other steps only requires triangle inequality and simple calculation.
\end{proof}

\subsection{The Proof of Lemma \ref{lem:anchoring}}
\begin{proof}
Let $z_g^*=(x_g^*,y_g^*)$ be the unique stationary point of $g(x,y)$ such that $G(z_g^*) = 0$. Then we have
\begin{align*}
    \Vert F(\tilde z) \Vert &\le \Vert G(\tilde z) \Vert + \lambda \Vert \tilde z - z_0 \Vert \\
    &\le   \Vert G(\tilde z) \Vert + \lambda \Vert \tilde z -z^*_g \Vert + \lambda \Vert z^*_g - z_0 \Vert  \\
    &\le 2 \Vert G(\tilde z) \Vert + \lambda \Vert z^* - z_0 \Vert,
\end{align*}
where the last step we use $G$ is $\lambda$-strongly-monotone and $\Vert z_g^* - z_0 \Vert \le \Vert z^* - z_0 \Vert$ from Lemma~\ref{lem:non-ex}.
\end{proof}

\section{The Proofs in Section \ref{sec:convex-concave}}

We provide detailed proofs for results in Section \ref{sec:convex-concave} except Lemma~\ref{lem:SEG}.
Note that Lemma~\ref{lem:SEG} is a special case of Theorem~\ref{thm:SEG-bias} with $b_{t}= b_{t+1/2} = 0$ and $\sigma_t = \sigma_{t+1/2} = \sigma$. 
Hence, it can be proved by directly using the analysis for Theorem~\ref{thm:SEG-bias} in Appendix~\ref{sec:proof-SEG-bias}.

\subsection{The Proof of Lemma \ref{lem:Epoch-SEG}}

\begin{proof}
Denote $\eta_k$ and $T_k$ be the step size and epoch length in the $k_{\rm th}$ epoch. From Lemma~\ref{lem:SEG} and strong monotone we know that
\begin{align} \label{eq:Epoch-SEG}
\mathbb{E}\Vert z_{k+1} - z^* \Vert^2 &\le \frac{1}{ \lambda \eta_k T_k} \mathbb{E}\Vert z_k - z^* \Vert^2 + \frac{16 \eta_k \sigma^2}{\lambda}. 
\end{align}
Telescoping from $k = 0,1,\cdots,N-1$, we obtain
\begin{align*}
\mathbb{E}\Vert z_N - z^* \Vert^2 \le \frac{1}{2^N} \mathbb{E}\Vert z_0 - z^* \Vert^2 + \frac{8 \sigma^2}{\lambda L}.
\end{align*}
Then we can use induction from $N,N+1,\cdots,N+K-1$ to show inequality in this theorem: suppose it is true for the case $N+k$, then for the case $N+k+1$, we have
\begin{align*}
& \mathbb{E}\Vert z_{N+k+1} - z^* \Vert^2 \\
\le &\frac{1}{4} \mathbb{E} \Vert z_{N+k} - z^* \Vert^2 + \frac{2 \sigma^2}{2^k \lambda L} \\
\le & \frac{1}{4} \left( \frac{1}{2^{N+2k} } \mathbb{E}\Vert z_0 - z^* \Vert^2 + \frac{8 \sigma^2}{2^k \lambda L}\right) + \frac{2 \sigma^2}{2^k \lambda L} \\
\le & \frac{1}{2^{N+2(k+1)}}\mathbb{E}\Vert z_0 - z^* \Vert^2 + \frac{8 \sigma^2}{2^{k+1} \lambda L},
\end{align*}
where the first inequality is by plugging $\eta_k,T_k$ into (\ref{eq:Epoch-SEG}) and the second one comes from the induction hypothesis.

In addition, as a side effect, it is clear that in this procedure, for any $k = 0,1,\cdots,N+K-1$, we can always maintain the bound of
\begin{align} \label{side:init-D}
    \BE\Vert z_k - z^* \Vert^2 \le \BE\Vert z_0 - z^* \Vert^2 + \frac{8 \sigma^2}{\lambda L}.
\end{align}
For each iteration of SEG, we need to quest the stochastic operator twice. Then summing up all the iterations in every epoch yields the complexity as claimed.
\end{proof}

\subsection{The Proof of Theorem \ref{thm:RAIN}}

\begin{proof}
By Lemma~\ref{lem:RAL} and taking expectation, we know that to find $z_S$ such that $\mathbb{E}\Vert F(z_S) \Vert \le \epsilon$ it is sufficient to guarantee
\begin{align} \label{eq:ms-RR}
    256  \lambda_s^2 S^2 \mathbb{E}\Vert z_{s+1} - z_s^* \Vert^2 \le \epsilon^2,
\end{align}
for $s = 0,1,\cdots,S-1$, where $z_s^*$ denotes the solution to 
\begin{align*}
\min_{x \in \BR^{d_x}} \max_{y \in \BR^{d_y}} f_s(x,y).    
\end{align*}

Then we prove (\ref{eq:ms-RR}) holds for $s = 0,1,\cdots,S-1$ by induction.
By the definition of $\lambda_s$, we know that $\lambda_{s+1} = 2 \lambda_s$ and each sub-problem $f_s(x,y)$ is at least $\lambda_s$-strongly-convex-$\lambda_s$-strongly-concave.
Now we specify the parameters in Epoch-SEG by
\begin{align*}
    z_{s+1} \leftarrow \text{Epoch-SEG}(f_s, z_s, \lambda_s, 2L, N_s, K_s).
\end{align*}
By results of (\ref{eq:Epoch-SEG}), we know that
\begin{align*}
&    \mathbb{E}\Vert z_{s+1} - z_s^* \Vert^2 \\
\le & \frac{1}{2^{N_s + 2K_s}} \mathbb{E}\Vert z_s - z_s^* \Vert^2+ \frac{4 \sigma^2}{2^{K_s} \lambda_s L} \\
\le & \frac{1}{2^{N_s + 2K_s}} \mathbb{E}\Vert z_s - z_{s-1}^* \Vert^2+ \frac{4 \sigma^2}{2^{K_s} \lambda_s L},
\end{align*}
where we use Lemma \ref{lem:anchoring} to obtain $\Vert z_s - z_s^* \Vert \le \Vert z_s - z_{s-1}^* \Vert$.

For the case of $s=0$, we have 
\begin{align*}
    \mathbb{E}\Vert z_1 - z_0^* \Vert^2 \le \frac{1}{2^{N_0}} \mathbb{E}\Vert z_0 - z^* \Vert^2 + \frac{4 \sigma^2}{2^{K_0} \lambda L},
\end{align*}
where we use $w_0^* = z^*$, $\lambda_0 = \lambda$ and $K_0 \ge 1$. Then we can let
\begin{align*}
N_0 \ge \log_2 \left(\frac{512 \lambda^2 S^2 D^2}{\epsilon^2} \right ) 
\quad\text{and}\quad    
2^{K_0} \ge \frac{2048 \lambda S^2 \sigma^2}{L \epsilon^2}
\end{align*}
to guarantee $\mathbb{E}\Vert z_1 - z_0^* \Vert^2$ meet our requirement.

Suppose we already have $256\lambda_{s-1}^2 S^2 \mathbb{E}\Vert z_s - z_{s-1}^* \Vert^2 \le \epsilon^2$. By observing that
\begin{align*}
\mathbb{E}[\Vert z_{s+1} - z_s^* \Vert^2] &\le \frac{1}{2^{N_s + 2K_s}} \mathbb{E}[\Vert z_s - z_s^* \Vert^2]+ \frac{4 \sigma^2}{2^{K_s} \lambda_s L} \\
&\le \frac{1}{2^{N_s + 2K_s}} \mathbb{E}[\Vert z_s - z_{s-1}^* \Vert^2]+ \frac{4 \sigma^2}{2^{K_s} \lambda_s L} \\
&\le \frac{1}{2^{K_s}} \times \frac{\epsilon^2}{256 \lambda_{s-1}^2 S^2} + \frac{4 \sigma^2}{2^{N_s} \lambda_s L},
\end{align*}
then let
\begin{align*}
    N_s \ge 3 \qquad \text{and} \qquad 2^{K_s} \ge \frac{2048 \lambda_s S^2 \sigma^2}{L \epsilon^2}
\end{align*}
for all $s\ge 1$ leads to 
\begin{align*}
256\lambda_s^2 S^2 \mathbb{E}\Vert z_{s+1} - z_s^* \Vert^2 \le \epsilon^2.
\end{align*}
The total SFO complexity is
\begin{align*}
& \sum_{s=0}^{S-1} \frac{2L}{\lambda_s} \times \left( 16 N_s + 64 \times  2^{K_s} \right) \\
=& \frac{2L N_0}{\lambda} + \sum_{s=1}^{S-1} \frac{96 L}{\lambda_s} + \sum_{s=0}^{S-1} \frac{128 L}{\lambda_s} \times 2^{K_s} \\
\le& \frac{2L N_0}{\lambda} + \sum_{s=1}^{S-1} \frac{96 L}{\lambda_s} + \sum_{s=0}^{S-1} \frac{512 L}{\lambda_s} \times \frac{2048 \lambda_s S^2 \sigma^2}{L \epsilon^2} \\
\le& \frac{2L N_0}{\lambda} + \frac{96L}{\lambda} + \frac{1048576 S^3 \sigma^2}{\epsilon^2}.
\end{align*}
Plugging the value of $N_0$ and $S$ into above equation completes the proof. 
\end{proof}

\subsection{The Proof of Theorem \ref{thm:EA-C}}

\begin{proof}
Let $w$ be the output of applying RAIN on the function $g(x,y)$ and we define the operator $G_0(z)\triangleq F(z)+\lambda(z-z_0)$.
By the anchoring lemma with any $0 < \lambda \le \epsilon/D$ and taking expectation, we know that if we can make sure $\mathbb{E}\Vert G_0(w) \Vert \le \epsilon$, then we can obtain  $\mathbb{E}\Vert F(w) \Vert \le 3 \epsilon$ . By Theorem \ref{thm:RAIN}, we can ensure $\mathbb{E}\Vert G(w) \Vert \le \epsilon$ within a SFO complexity of
\begin{align*}
\mathcal{O} \left( \frac{L+ \lambda}{\lambda} + \frac{L+ \lambda}{\lambda} \log \left( \frac{\lambda D}{\epsilon} \log \left ( \frac{L+\lambda}{\lambda}\right) \right) + \frac{\sigma^2}{\epsilon^2} \log^3 \left( \frac{L + \lambda}{\lambda}\right)\right).
\end{align*}
Then we complete the proof by plugging the value of $\lambda$ into above equation.
\end{proof}

\section{The Proofs in Section \ref{sec:extension}}

We provide the detailed proofs for results in Section~\ref{sec:extension}.

\subsection{The Proof of Proposition~\ref{prop:id-scsc}}
\begin{proof}
Since $\tau \ge 2L$, then we know
\begin{align*}
\nabla_{xx}^2 f(x,y) +  \nabla_{xy}^2 f(x,y)(2L I - \nabla_{yy}^2 f(x,y))^{-1} \nabla_{yx}^2 f(x,y) \succeq \alpha I
\end{align*}
and
\begin{align*}
-\nabla_{yy}^2 f(x,y) +  \nabla_{yx}^2 f(x,y)(2L I + \nabla_{xx}^2 f(x,y))^{-1} \nabla_{xy}^2 f(x,y) \succeq \alpha I,
\end{align*}
which implies the desired result by Proposition 2.6 of \citet{grimmer2020landscape}. \end{proof}

\subsection{The Proof of Proposition~\ref{prop:cc-envelope}}
\begin{proof}
According to the proof of Example 1 from \citet{lee2021fast}, we know that
\begin{align*}
    (F(z) - F(z'))^\top (z - z') \ge -\frac{1}{2L} \Vert F(z) - F(z') \Vert^2
\end{align*}
for any $z=(x,y)$ and $z'=(x',y')$, which implies the gradient operator of $f_{2L}(x,y)$ is monotone. 
It further indicates the saddle envelope $f_{2L}(x,y)$ is convex-concave \cite[Lemma~3]{liu2021first}.
\end{proof}

\subsection{The Proof of Theorem \ref{thm:Epoch-SEG+}}

\begin{proof}
We know that
\begin{align*}
\mathbb{P}(J = j) = 2^{-j}, \quad j = 1,\cdots,K.
\end{align*}
Since $\hat z_m$ is i.i.d for all $m$, we have
\begin{align*}
\mathbb{E}[\hat z] = \mathbb{E}[\hat z_0]=\mathbb{E}[z_N] + \sum_{j=1}^{K} \mathbb{P}(J = j) 2^j \mathbb{E}[z_{N+j} - z_{N+j-1}] = \mathbb{E}[z_{N+K}].
\end{align*}
Using  Lemma \ref{lem:Epoch-SEG} and Jensen's inequality, we obtain the upper bound for 
bias:
\begin{align*}
\Vert \mathbb{E}\hat z - z^*\Vert^2 \le \mathbb{E}\Vert z_{N+K} - z^* \Vert^2 \le \frac{1}{2^{N+2K} } \times \mathbb{E}\Vert z_0 - z^* \Vert^2 + \frac{8 \sigma^2}{2^K \lambda L};
\end{align*}
as well as the upper bound for variance: 
\begin{align*}
\begin{split}    
& \mathbb{E}\Vert \hat z - \mathbb{E} \hat z \Vert^2 \\
\le & \frac{1}{M} \mathbb{E}\Vert \hat z_0 - \mathbb{E} \hat z_0 \Vert^2 \\
\le & \frac{1}{M} \mathbb{E}\Vert \hat z_0 - z^* \Vert^2  \\
\le &\frac{2}{M} \mathbb{E}\Vert z_N - z^* \Vert^2 + \frac{2}{M} \mathbb{E}\Vert 2^J (z_{N+J} - z_{N+J-1}) \Vert^2 \\
=& \frac{2}{M} \mathbb{E}\Vert z_N - z^* \Vert^2 + \frac{2}{M} \sum_{j=1}^{K} \mathbb{P}(J=j) 2^j \Vert z_{N+j}  - z_{N+j-1} \Vert^2 \\
\le & \frac{2}{M} \mathbb{E}\Vert z_N - z^* \Vert^2 + \frac{4}{M} \sum_{j=1}^{K} 2^j \big(\Vert z_{N+j} - z^* \Vert^2 + \Vert z_{N+j-1} - z^* \Vert^2\big) \\
\le & \frac{2}{M} \mathbb{E}\Vert z_N - z^* \Vert^2 +  \frac{1}{M}\sum_{j=1}^{K}  \left\{\frac{1}{2^{N+j}} \times 20 \mathbb{E}\Vert z_0- z^* \Vert^2 + \frac{96 \sigma^2}{\lambda L} \right\} \\
\le & \frac{1}{M} \left\{\frac{2}{2^N}\mathbb{E}\Vert z_0 - z^* \Vert^2 + \frac{16 \sigma^2}{\lambda L}+ \frac{20}{2^N}\mathbb{E}\Vert z_0 - z^* \Vert^2 + \frac{96 K \sigma^2}{\lambda L} \right\} \\
\le & \frac{22}{2^N M} \mathbb{E}[\Vert z_0 - z^* \Vert^2] + \frac{112 K \sigma^2}{M \lambda L },
\end{split}
\end{align*}
where the second inequality follows from the fact that variance is always bounded by mean square error, i.e. it holds that
\begin{align*}
    & \mathbb{E}\Vert \hat z_0 - \mathbb{E} \hat z_0 \Vert^2 \\
    =& \mathbb{E}\Vert \hat z_0 \Vert^2 - \Vert \mathbb{E} \hat z_0 \Vert^2 \\
    \le& \mathbb{E}\Vert \hat z_0 \Vert^2 - 2  \langle \mathbb{E} \hat z_0, z^*\rangle + \Vert z^* \Vert^2 \\
    =& \mathbb{E}\Vert \hat z_0 - z^* \Vert^2;
\end{align*}
where the third inequality is from the definition of $\hat z_m$; the second last, third last and fourth last ones are all by Lemma \ref{lem:Epoch-SEG}; others are dependent on the Young's inequality or simple algebra.
The complexity in expectation can be derived by Lemma \ref{lem:Epoch-SEG}, which is
\begin{align*}
\quad & M \times \mathbb{E}\big[ 16\kappa N +  64 \kappa  2^J \big] \\
= & 16 \kappa MN + 64 \kappa M\sum_{j=1}^{K} \mathbb{P}(J=j) 2^j \\
= & 16 \kappa MN + 64 \kappa MK,
\end{align*}
where $\kappa=L/\lambda$.
\end{proof}

\subsection{The Proof of Theorem \ref{thm:SEG-bias}}\label{sec:proof-SEG-bias}

\begin{proof}
We begin from the strong monotonicity and an identity:
\begin{align} \label{eq:SEG-0}
\begin{split}
&  2\mathbb{E}[ F(z_{t+1/2})^\top (z_{t+1/2}- z^*) ] \\
=&  2\mathbb{E}[(F(z_{t+1/2}) - F(z_{t+1/2};\xi_j))^\top (z_{t+1/2}  -z^*) +F(z_{t+1/2};\xi_j) ^\top (z_{t+1} - z^*)] \\
& + 2\mathbb{E}[  ( F(z_{t+1/2};\xi_j) - F(z_t;\xi_i))^\top(z_{t+1/2} - z_{t+1}) +  F(z_{t};\xi_i) ^\top (z_{t+1/2} - z_{t+1})].
\end{split}
\end{align}
Note that we have
\begin{align} \label{eq:SEG-1}
\begin{split}
& 2\mathbb{E}[(F(z_{t+1/2}) - F(z_{t+1/2};\xi_j))^\top (z_{t+1/2}  -z^*)] \\
=& 2\mathbb{E}[F(z_{t+1/2}) - F(z_{t+1/2};\xi_j)]^\top \mathbb{E}[z_{t+1/2} - z^*] \\
\le & \frac{\lambda}{2} \mathbb{E}[ \Vert z_{t+1/2} - z^* \Vert^2] + \frac{2b_{t+1/2}^2}{\lambda},
\end{split}
\end{align}
where we use the independence and Young's inequality. Also, we have
\begin{align} \label{eq:SEG-2}
\begin{split}
& 2\mathbb{E}[F(z_{t+1/2};\xi_j) ^\top (z_{t+1} - z^*)] \\
=& \frac{2}{\eta} \mathbb{E}[ (z_t- z_{t+1})^\top (z_{t+1} - z^*)] \\
=& \frac{1}{\eta} \mathbb{E} [ \Vert z_t - z^* \Vert^2 - \Vert z_t - z_{t+1} \Vert^2 - \Vert z_{t+1} -z^* \Vert^2]. 
\end{split}
\end{align}
and
\begin{align} \label{eq:SEG-3}
\begin{split}
&\quad 2\mathbb{E}[F(z_{t};\xi_i) ^\top (z_{t+1/2} - z_{t+1})] \\
&= \frac{2}{\eta} \mathbb{E}[ (z_t- z_{t+1/2})^\top (z_{t+1/2} - z_{t+1})] \\
&= \frac{1}{\eta} \mathbb{E} [ \Vert z_t - z_{t+1} \Vert^2 - \Vert z_t - z_{t+1/2} \Vert^2 - \Vert z_{t+1/2} -z_{t+1} \Vert^2] \\
\end{split}
\end{align}
Using $\eta \le 1/(4 L)$ and  Young's inequality we obtain
\begin{align} \label{eq:SEG-4}
\begin{split}
&\quad 2\mathbb{E} [(F(z_{t+1/2};\xi_j) - F(z_t;\xi_i))^\top(z_{t+1/2} - z_{t+1})] \\
&\le \mathbb{E}\left[ 2\eta \Vert F(z_{t+1/2};\xi_j) - F(z_t;\xi_i) \Vert^2 + \frac{1}{2 \eta} \Vert z_{t+1/2} - z_{t+1} \Vert^2\right] \\
&\le \mathbb{E}\left[   6 \eta \Vert F(z_{t+1/2}) -  F(z_{t+1/2};\xi_j) \Vert^2  + 6 \eta \Vert F(z_{t})- F(z_{t};\xi_i) \Vert^2 \right]\\ 
&\quad +\mathbb{E} \left[ 6 \eta \Vert F(z_{t+1/2}) - F(z_t) \Vert^2  +\frac{1}{2 \eta} \Vert z_{t+1/2} - z_{t+1} \Vert^2\right] \\
&\le 6 \eta (e_t^2 + e_{t+1/2}^2) + 6 \eta L^2 \mathbb{E}[  \Vert z_t - z_{t+1/2} \Vert^2] + \frac{1}{2 \eta} \mathbb{E}[\Vert z_{t+1/2 } - z_{t+1 } \Vert^2] \\
&\le  6 \eta (e_t^2 + e_{t+1/2}^2) + \frac{1}{2 \eta} \mathbb{E}[  \Vert z_t - z_{t+1/2} \Vert^2+\Vert z_{t+1/2 } - z_{t+1} \Vert^2], 
\end{split}
\end{align}
where the second last inequality follows from $L$-Lipschitz property.

The final step is to plug (\ref{eq:SEG-1}) (\ref{eq:SEG-2}) (\ref{eq:SEG-3}) (\ref{eq:SEG-4}) into (\ref{eq:SEG-0}) and then apply the strongly monotonicity which implying the fact that
\begin{align*}
    F(z_{t+1/2})^\top (z_{t+1/2} - z^*) \ge  \lambda \Vert z_{t+1/2} - z^* \Vert^2.
\end{align*}
\end{proof}

\subsection{The Proof of Theorem \ref{thm:RAIN++-ID} and Theorem \ref{thm:RAIN++-NC} }

First of all, we present some useful lemmas.
\begin{lem}[non-expansiveness of the resolvent]\label{lem:NER}
Denote
\begin{itemize}
    \item $z^+=(x^+,y^+)$ as the solution to $\min_{x' \in \BR^{d_x}} \max_{y' \in \BR^{d_y}} f(x',y') + L \Vert x' - x \Vert^2 - L \Vert y' - y \Vert^2$;
    \item $w^+=(u^+,v^+)$ as  the solution to $\min_{x' \in \BR^{d_x}} \max_{y' \in \BR^{d_y}} f(x',y') + L \Vert x' - u \Vert^2 - L \Vert y' - v \Vert^2$.
\end{itemize}
Then it holds that $\Vert z^+ - w^+ \Vert \le 2\Vert z - w \Vert$, where $z = (x,y)$ and $w = (u,v)$.
\end{lem}
\begin{proof}
It follows from 
\begin{align*}
    \Vert z^+ - w^+ \Vert \le \Vert z - w \Vert + \frac{1}{2L} \Vert F(z^+) - F(w^+) \Vert \le \Vert z - w \Vert + \frac{1}{2} \Vert z^+ - w^+ \Vert,
\end{align*}
where the relationship $F(z^+) = 2L(z - z^+)$ and $F(w^+) = 2L(w - w^+)$ and the triangle inequality is used. 
\end{proof}

\begin{lem} \label{lem:6L}
If $F$ is a $L$-Lipschitz continuous operator, then $F_{2L}$ is a $6L$-Lipschitz continuous operator.
\end{lem}

\begin{proof}
Note that we have the following relationship of
\begin{align*}
F_{2L} (z) = F(z^+) = 2L (z  - z^+).
\end{align*}
Hence we know that the $F_{2L}(\,\cdot\,)$ is $6L$-Lipschitz continuous by noting
\begin{align*}
\Vert F(z^+) - F(w^+) \Vert \le 2 L \Vert z - w \Vert + 2L \Vert z^+ - w^+ \Vert \le 6 L \Vert z - w \Vert
\end{align*}
where we use Lemma~\ref{lem:NER} to obtain that $\Vert z^+ - w^+ \Vert \le 2\Vert z - w \Vert$.
\end{proof}



As we have mentioned, for given $k,s$, the key is to control the bias as well as the variance term of the estimated gradient operator on envelope in SEG$^{++}$ (Algorithm \ref{alg:SEG++}) within $\tilde \fO(1)$ SFO complexity. We present the details in the following theorem. 

\begin{thm} \label{thm:SEG++}
Under the setting of both Theorem \ref{thm:RAIN++-ID} and Theorem \ref{thm:RAIN++-NC}, if the input of \SEGpp (Algorithm \ref{alg:SEG++}) holds that
\begin{align} \label{eq:uni-bound}
    \lambda \eta \ge \delta_{\rm min}, \quad T \le  T_{\rm max}, \quad \Vert z_s^* - z_0 \Vert \le D_{\rm max}, \quad \Vert w_0 - z_0^+ \Vert^2 \le \frac{\sigma^2}{4L^2} + D^2
\end{align}
for some $D_{\rm max} \ge D$, {$\delta_{\min} >0$, and, $T_{\rm max} >0$} and we set
\begin{align} \label{value:NMK}
    K = \left \lceil \log_2 \left(\frac{12}{\delta_{\rm min}} \right)\right \rceil,\quad
    N =  \left \lceil \log_2 \left(\max \left \{ \frac{2 T_{\rm max}}{\delta_{\rm min}}, \frac{8 D_{\rm max}^2 L^2}{\delta_{\rm min} \sigma^2} \right\} \right)\right \rceil,\quad M = 1792 K,
\end{align}
then it guarantees
\begin{align} \label{eq:SEG++}
\begin{split}
 \lambda \mathbb{E}[\Vert z_{t+1/2} - z_s^* \Vert^2] 
\le  \frac{1}{\eta} \mathbb{E}[\Vert z_{t} - z_s^* \Vert^2 - \Vert z_{t+1} - z_s^* \Vert^2]  - \frac{1}{2 \eta} \mathbb{E}[\Vert z_{t+1} - z_{t+1/2} \Vert^2] + 16 \eta \sigma^2
\end{split}
\end{align}
for any $s = 0,1,\cdots, S-1$, where $z_s^*=(x_s^*,y_s^*)$ is the solution to 
\begin{align*}
\min_{x \in \BR^{d_x} } \max_{y\in \BR^{d_y}} f_{2L}^{(s)}(x,y)    
\end{align*}
and $f_{2L}^{(s)}(x,y)$ is defined in (\ref{eq:env-s}) for ID case or (\ref{eq:env-s-nc}) for NC case, respectively. 
Furthermore, the output holds that
\begin{align} \label{sigma:output}
\BE\Vert \hat z_{J+1/2} - z_{J+1/2}^+   \Vert^2 \le \frac{\sigma^2}{4L^2},
\end{align}
where $z^+=(x^+, y^+)$ is the solution to $ \min_{x' \in \BR^{d_x}} \max_{y' \in \BR^{d_y}} f(x,y) + L \Vert x' - x \Vert^2 - L \Vert y' - y \Vert^2$;
and each call of \ESEGp in \SEGpp  (line 6 and line 10 of Algorithm \ref{alg:SEG++}) can be finished within the SFO complexity of
\begin{align} \label{log:SFO}
\fO \left( \max \left\{ \log \left(\frac{T_{\rm max}}{\delta_{\rm min}}\right),\, \log \left( \frac{D_{\rm max}^2 L^2}{\delta_{\rm min}\sigma^2} \right) \right\}  \log \left( \frac{1}{\delta_{\rm min}}\right) \right).
\end{align}

\end{thm}

\begin{proof}
We introduce the following notations for our proof:
\begin{itemize}
\item We denote $F_{2L}^{(s)}$ as the gradient operator of $f_{2L}^{(s)}(x,y)$.
\item We denote $z^+=(x^+,y^+)$ as the solution to the minimax problem
\begin{align*}
\min_{x' \in \BR^{d_x}} \max_{y' \in \BR^{d_y}} f(x',y') + L \Vert x' - x \Vert^2 - L \Vert y' - y \Vert^2.
\end{align*}

\item In the procedure of Algorithm \ref{alg:RAIN++:expand}, we denote 
\begin{align*}
\hat F_{2L}^{(s)}(z) = 
\begin{cases}
2L(z - \hat z^+) + \sum_{i=1}^s \lambda_i (z - z_i), & \text{ ID case}; \\
2L(z - \hat z^+) + \sum_{i=0}^s \lambda_i (z - z_i), & \text{ NC case};
\end{cases}
\end{align*}
for given $z=(x,y)$ and $s$. Note that $\hat F_{2L}^{(s)}$ is used to approximate $F_{2L}^{(s)}$ in our algorithm and analysis.
\item We denote 
{\small \begin{align*}
& b_t^{(s)} \triangleq \Vert \mathbb{E} \hat F_{2L}^{(s)}(z_t) - F_{2L}^{(s)}(z_t) \Vert,~~~~~~~~~~~~~~~~~~
\big(\sigma_{t}^{(s)}\big)^2 \triangleq \mathbb{E} [\Vert \hat F_{2L}^{(s)}(z_t) - \mathbb{E} \hat F_{2L}^{(s)}(z_{t}) \Vert^2], \\
& b_{t+1/2}^{(s)} \triangleq \Vert \mathbb{E} \hat F_{2L}^{(s)}(z_{t+1/2}) - F_{2L}^{(s)}(z_{t+1/2}) \Vert,~~~~
\big(\sigma_{t+1/2}^{(s)}\big)^2 \triangleq \mathbb{E} [\Vert \hat F_{2L}^{(s)}(z_{t+1/2}) - \mathbb{E} \hat F_{2L}^{(s)}(z_{t+1/2}) \Vert^2 ].
\end{align*}}
for the bias and variance of the approximate gradient operator $F_{2L}^{(s)}$. \end{itemize}
If the the bias and variance of $F_{2L}^{(s)}(\cdot)$ satisfy
\begin{align}\label{cond:bias-variance}
b_t^{(s)}\leq 2\lambda\eta\sigma^2,\quad b_{t+1/2}^{(s)}\leq 2\lambda\eta\sigma^2,\quad
\big(\sigma_t^{(s)}\big)^2\leq\frac{\sigma^2}{2} 
\quad \text{and} \quad
\big(\sigma_{t+1}^{(s)}\big)^2\leq \frac{\sigma^2}{2},
\end{align}
then applying Theorem~\ref{thm:SEG-bias} on $f_{2L}^{(s)}$ leads to the result of (\ref{eq:SEG++}). 
The Lipschitz continuity of $\hat F_{2L}$ indicates the conditions (\ref{cond:bias-variance}) holds if we can prove the following claims:
\begin{align} \label{eq:induction}
\begin{split}
&\texttt{Claim I}:~~  4L^2 \Vert \mathbb{E} \hat z_t^+ - z_t^+ \Vert^2 \le 2 \lambda \eta \sigma^2,~~~~~~~~~~~~~4L^2 \mathbb{E}[\Vert \hat z_t^+ - \mathbb{E} \hat z_t^+ \Vert^2] \le \sigma^2/2; \\
&\texttt{Claim II}:~ 4L^2 \Vert \mathbb{E} \hat z_{t+1/2}^+ - z_{t+1/2}^+ \Vert^2 \le 2 \lambda \eta\sigma^2,~~~~4L^2 \mathbb{E}[\Vert \hat z_{t+1/2}^+ - \mathbb{E}\hat z_{t+1/2}^+ \Vert^2] \le \sigma^2/2;
\end{split}
\end{align}
Now start to prove (\ref{eq:induction}) holds for all $t$ by induction, which implies the result of (\ref{eq:SEG++}) in the theorem.
\paragraph{Induction Base:} 
For Claim I with $t=0$, applying Theorem \ref{thm:Epoch-SEG+} and using the fact $\hat z_{-1/2}^+ = w_0$, we have
\begin{align*}
    \Vert \BE \hat z_0^+ - z_0^+ \Vert^2 &\le \frac{1}{2^{N+2K}} \times  \Vert w_0 - z_0^+ \Vert^2 + \frac{8 \sigma^2}{2^K \times 3L^2} 
\end{align*}
and
\begin{align*}
    \BE \Vert \hat z_0^+ - \BE \hat z_0^+ \Vert^2 &\le \frac{1}{2^N M} \times 22 \Vert w_0 - z_0^+ \Vert^2 + \frac{112K \sigma^2}{M \times 3L^2}.
\end{align*}
Plugging the setting of $N,M$ and $K$ as (\ref{value:NMK}) into above inequalities and using the bound $\Vert w_0 - z_0^+ \Vert \le \sigma^2 / (4L^2) + D^2$ followed from the assumption, we obtain Claim I for $t = 0$.
The next step should be showing Claim II holds for $t=0$. We can prove this result by the same way as proving \texttt{Claim I} $\Rightarrow$ \texttt{Claim II}, which will be detailed presented in upcoming paragraph.
\paragraph{Induction Step:}  Suppose the results of (\ref{eq:induction}) hold for all $t' \le t$, then we target to show \texttt{Claim I} and \texttt{Claim II} both hold for $t+1$. 

We first consider \texttt{Claim I}. Since we already have (\ref{eq:induction}) for all $t' \le t$, then it implies that (\ref{eq:SEG++}) also holds for all $t' \le t$ by using Theorem \ref{thm:SEG-bias}. Telescoping the result of (\ref{eq:SEG++}) for $t' = 0,\cdots, t$, we obtain
\begin{align}\label{bound:zt12-zt1}
\mathbb{E}\Vert z_{t+1/2} - z_{t+1} \Vert^2 \le 2\mathbb{E}\Vert z_0 - z_s^* \Vert^2 + 32 \eta^2 \sigma^2 T_{\rm max}
\end{align}
and
\begin{align} \label{bound:zt+1s}
\mathbb{E}\Vert z_{t+1} - z_s^* \Vert^2 \le \mathbb{E}\Vert z_0- z_s^* \Vert^2 + 16 \eta^2 \sigma^2 T_{\rm max}.
\end{align}
Using the bias-variance decomposition, i.e.
\begin{align*}
    \mathbb{E}\Vert \hat z_{t+1/2}^+ - z_{t+1/2}^+ \Vert^2 = \Vert z_{t+1/2}^+ - \BE\hat z_{t+1/2}^+ \Vert^2 + \BE\Vert \hat z_{t+1/2}^+ - \BE \hat z_{t+1/2}^+ \Vert^2
\end{align*}
and the induction hypothesis, we know that
\begin{align*}
    4L^2 \mathbb{E}\Vert \hat z_{t+1/2}^+ - z_{t+1/2}^+ \Vert^2 \le \sigma^2.
\end{align*} 
Then we have
\begin{align} \label{eq:bound-cmp}
\begin{split} 
&\quad \mathbb{E}\Vert \hat z_{t+1/2}^+ -z_{t+1}^+ \Vert^2  \\
&\le 2 \mathbb{E}\Vert \hat z_{t+1/2}^+ - z_{t+1/2}^+ \Vert^2 + 2 \mathbb{E}\Vert z_{t+1/2}^+ - z_{t+1}^+ \Vert^2 \\
&\le 2 \mathbb{E}\Vert \hat z_{t+1/2}^+ - z_{t+1/2}^+ \Vert^2 + 8 \mathbb{E}\Vert z_{t+1/2} - z_{t+1} \Vert^2 \\
&\le \frac{\sigma^2}{2L^2} + 16\mathbb{E}\Vert z_0 - z_s^* \Vert^2 + 256 \eta^2 \sigma^2 T_{\rm max} \\
&\le \frac{ \sigma^2 T_{\rm max}}{L^2} + 16\mathbb{E}\Vert z_0 - z_s^* \Vert^2,
\end{split}
\end{align}
where we use the bound for $\mathbb{E}\Vert \hat z_{t+1/2}^+ - z_{t+1/2}^+ \Vert$ by induction hypothesis and the bound of $\mathbb{E}\Vert z_{t+1/2} - z_{t+1} \Vert^2$ shown in (\ref{bound:zt12-zt1}).

Note that the algorithm apply the update rule
\begin{align*}
    \hat z_{t+1}^+ \leftarrow \text{Epoch-SEG}^+ (g_{t+1},\hat z_{t+1/2}^+,L, 3L,N,K,M).
\end{align*}
Applying Theorem \ref{thm:Epoch-SEG+}, we have
\begin{align*}
&\quad \Vert z_{t+1}^+ - \mathbb{E} \hat z_{t+1}^+ \Vert^2 \\
&\le \frac{1}{2^{N+2K} } \times \mathbb{E}\Vert \hat z_{t+1/2}^+ - z_{t+1}^+ \Vert^2 + \frac{8 \sigma^2}{2^K \times 3 L^2} \\
&\le \underbrace{\frac{1}{2^{N+2K} } \times \left( \frac{ \sigma^2 T_{\rm max}}{L^2} +16 \mathbb{E}\Vert z_0 - z_s^* \Vert^2 \right)}_{A_1} +  \underbrace{\frac{8 \sigma^2}{2^K \times3 L^2}}_{A_2}.
\end{align*}
and
\begin{align*}
\begin{split}    
&\quad \mathbb{E}\Vert \hat z_{t+1}^+ - \mathbb{E} \hat z_{t+1}^+ \Vert^2 \\
& \le \frac{1}{2^N M}\times 22 \mathbb{E}\Vert \hat z_{t+1/2}^+ - z_{t+1}^+ \Vert^2 + \frac{112K \sigma^2}{M \times 3 L^2 } \\
&\le \underbrace{\frac{1}{2^N M}\times \left( \frac{22 \sigma^2 T_{\rm max}}{L^2} + 352 \mathbb{E}\Vert z_0 - z_s^* \Vert^2 \right)}_{B_1} + \underbrace{\frac{112 K \sigma^2}{M \times 3 L^2}}_{B_2}. 
\end{split}
\end{align*}
The parameters setting as (\ref{value:NMK}) guarantees $A_1,A_2,B_1$ and $B_2$ are sufficient small, which leads to
\begin{align} \label{eq:para1-SEG++}
\begin{split}
4L^2 \Vert \mathbb{E} \hat z_{t+1}^+ - z_{t+1}^+ \Vert^2 &\le 2 \lambda \eta\sigma^2  
\qquad\text{and}\qquad
4L^2 \mathbb{E}[\Vert \hat z_{t+1}^+ - \mathbb{E} \hat z_{t+1}^+ \Vert^2] \le \frac{\sigma^2}{2}.
\end{split}
\end{align}
Then we finish the proof of \texttt{Claim I} for $t+1$. 

Now we show that \texttt{Claim I} $\Rightarrow$ \texttt{Claim II} to finish our induction. We begin from
\begin{align} \label{bound:zt+32t+1}
\begin{split}
&\quad \mathbb{E}\Vert z_{t+3/2} - z_{t+1} \Vert^2 \\
&= \eta^2 \mathbb{E}\Vert \hat F_{2L}^{(s)} (z_{t+1}) \Vert^2 \\
&\le 2 \eta^2 \mathbb{E}\Vert F_{2L}^{(s)}(z_{t+1}) \Vert^2 + 2 \eta^2 \mathbb{E}\Vert F_{2L}^{(s)}(z_{t+1}) - \hat F_{2L}^{(s)}(z_{t+1}) \Vert^2 \\
&\le 1152 \eta^2 L^2 \mathbb{E}\Vert z_{t+1} - z_s^* \Vert^2 + 2 \eta^2 \sigma^2 \\
&\le 1152 \eta^2 L^2\left( \mathbb{E}\Vert z_0 - z_s^* \Vert^2 +16 \eta^2 \sigma^2 T_{\rm max}\right) +  2 \eta^2 \sigma^2,
\end{split}
\end{align}
where the equality is based on the update $z_{t+3/2} \leftarrow z_{t+1} -  \eta \hat F_{2L}^{(s)}(z_{t+1})$; the inequalities use the fact that $F_{2L}^{(s)}$ is $24L$-Lipschitz continuous and the upper bound of $\mathbb{E}\Vert z_{t+1} - z_s^* \Vert^2$ which we have shown in (\ref{bound:zt+1s}).\footnote{Note that in the case of $t = -1$, bound of $\BE[ \Vert z_{t+1} - z_s^* \Vert^2]$ is independent on the induction hypothesis.}

Furthermore, the results of (\ref{eq:para1-SEG++}) imply
\begin{align*}
    4L^2\mathbb{E}\Vert z_{t+1}^+ - \hat z_{t+1}^+ \Vert^2 \le \sigma^2.
\end{align*}
Combining Lemma~\ref{lem:NER} with the bound of $\mathbb{E}\Vert z_{t+3/2} - z_{t+1} \Vert^2$ as shown in (\ref{bound:zt+32t+1}), we obtain
\begin{align*}
 & \mathbb{E}\Vert z_{t+3/2}^+ - \hat z_{t+1}^+ \Vert^2 \\
\le & 2 \mathbb{E}\Vert z_{t+1}^+ - \hat z_{t+1}^+ \Vert^2 + 2\mathbb{E}\Vert z_{t+3/2}^+ - z_{t+1}^+ \Vert^2 \\
\le & 2 \mathbb{E}\Vert z_{t+1}^+ - \hat z_{t+1}^+ \Vert^2  + 8\mathbb{E}\Vert z_{t+3/2} - z_{t+1} \Vert^2 \\
\le & \frac{\sigma^2 T_{\rm max}}{L^2} + 16\mathbb{E}\Vert z_0 - z_s^*\Vert^2,
\end{align*}
which is derived by the similar way to the proof of (\ref{eq:bound-cmp}). 
Note that the algorithm use the update
\begin{align*}
    \hat z_{t+3/2}^+ \leftarrow \text{Epoch-SEG}^+ (g_{t+1}, \hat z_{t+1}^+, L, 3L, N,K,M).
\end{align*}
Applying Theorem \ref{thm:Epoch-SEG+}, we have
\begin{align*}
& \Vert \mathbb{E}\hat z_{t+3/2}^+ - z_{t+3/2}^+\Vert^2 \\
\le & \frac{1}{2^{N+2K} } \times \mathbb{E}\Vert \hat z_{t+1}^+ - z_{t+3/2}^+ \Vert^2 + \frac{8 \sigma^2}{2^K \times 3 L^2} \\
\le & \underbrace{\frac{1}{2^{N+2K}} \times \left( \frac{ \sigma^2 T_{\rm max}}{L^2}  +16 \mathbb{E}\Vert z_0 - z_s^* \Vert^2 \right) }_{A_1}+  \underbrace{\frac{8 \sigma^2}{2^K \times3 L^2}}_{A_2}.
\end{align*}
and
\begin{align*}
&\quad \mathbb{E}\Vert \hat z_{t+3/2}^+ - \mathbb{E} \hat z_{t+3/2}^+ \Vert^2 \\
& \le \frac{1}{2^N M}\times 22 \mathbb{E}\Vert \hat z_{t+1}^+ - \hat z_{t+3/2}^+ \Vert^2 + \frac{112K \sigma^2}{M \times 3 L^2 } \\
&\le \underbrace{\frac{1}{2^N M}\times \left( \frac{22 \sigma^2 T_{\rm max}}{L^2} + 352 \mathbb{E}\Vert z_0 - z_s^* \Vert^2 \right)}_{B_1} + \underbrace{\frac{112 K \sigma^2}{M \times 3 L^2}}_{B_2}. 
\end{align*}
The parameters setting of (\ref{value:NMK}) guarantees
\begin{align} \label{eq:para2-SEG++}
\begin{split}
4L^2 \Vert \mathbb{E} \hat z_{t+3/2}^+ - z_{t+3/2}^+ \Vert^2 &\le 2 \lambda \eta\sigma^2  \qquad \text{and} \qquad
4L^2 \mathbb{E}\Vert \hat z_{t+3/3}^+ - \mathbb{E}\hat z_{t+3/2}^+ \Vert^2 \le \frac{\sigma^2}{2},
\end{split}
\end{align}
which is means \texttt{Claim II} holds. Hence, we have complete the induction.

We can verify the the condition number of sub-problem  is no more than $3$: since it is $L$-strongly-convex-$L$-strongly-concave and $3L$-smooth. Then based on the setting of (\ref{value:NMK})
and Theorem \ref{thm:Epoch-SEG+}, it guarantees that the SFO complexity of each call of \ESEGp (line 6 and line 10 of Algorithm \ref{alg:SEG++}) is no more than
\begin{align*}
\quad 48 MN + 192 MK 
= \fO\left(\max \left\{ \log \left(\frac{T_{\rm max}}{\delta_{\rm min}}\right) ,\log \left( \frac{D_{\rm max}^2 L^2}{\delta_{\rm min}\sigma^2} \right) \right\}\log\left( \frac{1}{\delta_{\rm min}}\right)\right).
\end{align*}
\end{proof}

\subsubsection{The Proof of Theorem \ref{thm:RAIN++-ID} and Theorem \ref{thm:RAIN++-NC}}

Note that the parameters $N_s$ and $K_s$ in Theorem \ref{thm:RAIN++-ID} follow (\ref{eq:para-EA}) by replacing $L$ with $6L$. 
Our proof of Theorem \ref{thm:RAIN++-ID} and Theorem \ref{thm:RAIN++-NC} will be described by the notation of Algorithm~\ref{alg:RAIN++:expand}.

\begin{proof}
The main steps of \RAINpp~(Algorithm~\ref{alg:RAIN++:expand}) are based on calling \SEGpp~ with $s=0, 1,\dots,S-1$ and $k = 0, 1,\dots, N_s + K_s-1$.
By Theorem~\ref{thm:SEG++}, once the initial conditions of (\ref{eq:uni-bound}) hold for the call of \SEGpp ~in  \RAINpp~(Algorithm~\ref{alg:RAIN++:expand}), we achieve the bound (\ref{eq:SEG++}) on $f_{2L}^{(s)}(x,y)$ (which is defined in (\ref{eq:env-s}) for ID case and (\ref{eq:env-s-nc}) for NC case) for any $s$ (just like running SEG on $f_s$).
Then the SFO complexity of each \SEGpp~ call in \RAINpp is
\begin{align} \label{eq:log}
l = \fO \left(\max \left\{ \log \left(\frac{T_{\rm max}}{\delta_{\rm min}}\right) ,\log \left( \frac{D_{\rm max}^2 L^2}{\delta_{\rm min}\sigma^2} \right) \right\}  \log \left( \frac{1}{\delta_{\rm min}}\right) \right).
\end{align}
Note that \RAINpp~can be regarded as the modification of RAIN by replacing the subroutines SEG with \SEGpp.
Compared with SEG, \SEGpp~requires additional stochastic oracle calls to estimate $F(\cdot)$. 
Therefore, the total SFO complexity of RAIN$^{++}$ is the total SFO complexity of RAIN (shown in Theorem \ref{thm:RAIN}) multiplying $l$, which implies the SFO complexity of $\tilde \fO( \sigma^2 \epsilon^{-2} + L/\lambda)$ we desired.
Hence, the remains in the proof only need to verify there exist some $T_{\rm max}$, $\delta_{\rm min}$ and $D_{\rm max}$  to guarantee the conditions of (\ref{eq:uni-bound}) hold for each call of \SEGpp~in \RAINpp~(line 5 and line 7 in Algorithm~\ref{alg:RAIN++:expand}).

\paragraph{The upper bound of $\Vert w_{s,k} - z_{s,k}^+ \Vert^2$:} Consider the first time when SEG$^{++}$ (Algorithm \ref{alg:SEG++}) is called in Algorithm~\ref{alg:RAIN++:expand} (i.e. $s = 0$ and $k = 0$). 
Note that we have $w_{s,k} = z_{s,k} = z_0$, then Lemma~\ref{lem:NER} means $\Vert z_0^+ - z^* \Vert \le 2\Vert z_0 - z^* \Vert$ and we achieve
{\small \begin{align*}
    \Vert  w_{s,k} - z_{s,k}^+ \Vert = \Vert z_0 - z_0^+ \Vert = \frac{1}{2L} \Vert F(z_0^+) \Vert = \frac{1}{2L} \Vert F(z_0^+) - F(z^*) \Vert \le \frac{1}{2}\Vert z_0^+ - z^* \Vert \le  \Vert z_0 - z^* \Vert \le D.
\end{align*}}
Now we consider the other cases (i.e. $s>0$ or $k>0$). In these rounds, SEG$^{++}$ use $(z_{s,k},w_{s,k})$ as the initial point, which is the output of the previous call of SEG$^{++}$. Then it holds that 
\begin{align*}
\BE\Vert z_{s,k}^+ - w_{s,k} \Vert^2 \le \frac{\sigma^2}{4L^2}
\end{align*}
by the result of (\ref{sigma:output}) in Theorem~\ref{thm:SEG++}. Therefore, combining above the two cases, it holds that 
\begin{align*}
    \BE\Vert z_{s,k}^+ - w_{s,k} \Vert^2 \le \frac{\sigma^2}{4L^2} + D^2
\end{align*}
for all $s$ and $k$.

\paragraph{The settings of $T_{\rm max}$ and $1/\delta_{\rm min}$:} We present a simple bound for $T_{\rm max}$ and $1/\delta_{\rm min}$ from (\ref{eq:para-EA}) and (\ref{eq:ms-RR}). 
For each stage of $s$ in Algorithm~\ref{alg:RAIN++:expand}, we run SEG$^{++}$ for two phases, that is $k=0,\dots,N_{s-1}$ and $k=N_s\dots,N_s+K_s+1$. 
\begin{itemize}
\item For $k=0,\dots,N_{s-1}$, we run SEG$^{++}$ by the stepsize of $1/(48 L)$ and the iteration number of $96L/\lambda_s$. Since $\lambda_s \ge \lambda$, setting $\delta_{\rm min}$ and $T_{\rm max}$ with
\begin{align}\label{bound:T_max-delta_min}
\frac{1}{\delta_{\rm min}} \le \frac{48 L}{\lambda} \qquad\text{and}\qquad \quad T_{\rm max} \le \frac{96 L}{\lambda} 
\end{align}
satisfies the conditions in (\ref{eq:uni-bound}). 
\item For $k=N_s\dots,N_s+K_s+1$, the stepsize of \SEGpp~is decreasing and the iteration numbers of \SEGpp~is decreasing. 
So we only needs to consider the last time we call SEG$^{++}$, whose
stepsize and iteration numbers are $1/(96 L \times 2^{K_s -N_s})$ and $384 \times2^{K_s-N_s}$ respectively.
Since we have $\lambda_s \le 6 L$, setting
\begin{align}\label{bound:T_max-delta_min2}
T_{\rm max} &\le 384 \times2^{K_{S-1}- N_{S-1}} \le 786432 S^2 \sigma^2 \epsilon^{-2},
\end{align}
and
\begin{align}\label{bound:T_max-delta_min3}
\frac{1}{\delta_{\rm min}} &\le \frac{1}{\lambda_{S-1}/(96 L \times 2^{K_{S-1}- N_{S-1}})} \le  8192 S^2 \sigma^2 \epsilon^{-2}.
\end{align}
satisfies the conditions in (\ref{eq:uni-bound}). 
\end{itemize}
Combining the bounds of (\ref{bound:T_max-delta_min}), (\ref{bound:T_max-delta_min2}) and (\ref{bound:T_max-delta_min3}), we obtain
\begin{align} \label{setting-large}
\begin{split}
\frac{1}{\delta_{\rm min}} &\le \max\left\{ 48 L /\lambda , 8192 S^2 \sigma^2 \epsilon^{-2} \right\} \quad \text{and}  \\
T_{\rm max} &\le \max \left\{ 96 L / \lambda, 786432 S^2 \sigma^2 \epsilon^{-2} \right\}.
\end{split}
\end{align}

\paragraph{The setting of $D_{\rm max}$:} This paragraph shows for all $s$ and $k$, the term $\BE[\Vert z_{s,k} - z_s^* \Vert^2]$ can be bounded by some positive constant $D_{\rm max}^2$. Since Line 4-8 in Algorithm~\ref{alg:RAIN++:expand} can be regarded as running Epoch-SEG on $f_{2L}^{(s)}$, the result (\ref{side:init-D}) in the proof of Lemma~\ref{lem:Epoch-SEG} means
\begin{align} \label{bound:sk-star}
    \BE\Vert z_{s,k} - z_s^* \Vert^2 \le \BE\Vert z_{s,0} - z_s^*  \Vert^2 + \frac{8 \sigma^2}{\lambda L}
\end{align}
for any $k =0,1,\cdots, N_s+K_s-1$. Now we bound $\BE\Vert z_{s,0} - z_s^* \Vert^2$. 
\begin{itemize}
\item For $s = 0$, we have $z_{s,0} = z_0$ and $z_s^* = z^*$, which directly implies $\BE\Vert z_{s,0} - z_s^* \Vert^2 = D^2$. 
\item For $s\ge 1$, the result of (\ref{eq:ms-RR}) guarantees
\begin{align*} 
    256  \lambda_{s-1}^2 S^2 \mathbb{E}\Vert z_{s,0} - z_{s-1}^* \Vert^2 \le \epsilon^2.
\end{align*}
Using the non-expansiveness after anchoring (Lemma \ref{lem:non-ex}), we obtain
\begin{align*}
    \BE\Vert z_{s,0}- z_{s}^* \Vert^2 \le \mathbb{E}\Vert z_{s,0} - z_{s-1}^* \Vert^2 \le \frac{\epsilon^2}{256 \lambda^2 S^2}.
\end{align*}
for any $s \ge 1$.
\end{itemize}
Finally, combining above two cases and (\ref{bound:sk-star}) means the setting
\begin{align}\label{eq:Dmax}
    D_{\rm max}^2 \le \max \left\{ D^2, \frac{\epsilon^2}{256 \lambda^2 S^2} \right\} + \frac{8 \sigma^2}{\lambda L}.
\end{align}
satisfies the condition in (\ref{eq:uni-bound}).

{In summary, we can set $\delta_{\rm min}$, $T_{\rm max}$ and $D_{\rm max}$ by following (\ref{setting-large}) and (\ref{eq:Dmax}) to guarantee the conditions in (\ref{eq:uni-bound}) hold, 
then the SFO complexity of each SEG$^{++}$ call in RAIN$^{++}$ is $l = {\rm polylog} (L/\lambda, 1/ \epsilon, D, \sigma)=\tilde \fO(1)$ and total SFO complexity of RAIN$^{++}$ is $\tilde \fO(\sigma^2 \epsilon^{-2} + L / \lambda )$. Since Proposition~\ref{prop:id-scsc} says $\lambda = \Theta(\alpha)$, the SFO complexity is $\tilde\fO(\sigma^2 \epsilon^{-2} + L / \alpha )$ for ID case. 
For the NC case, we takes $\lambda  = \Theta(\epsilon)$, which leads to the SFO complexity of $\tilde\fO(\sigma^2 \epsilon^{-2} + L / \epsilon )$.}
\end{proof}

\subsection{The Proof of Proposition~\ref{prop:final}}
\begin{proof}
We denote $F_{2L}$ as the gradient operator of $f_{2L}(x,y)$ and $F$ as the gradient operator of $f(x,y)$. Let $z^+=(x^+,y^+)$ be the solution to 
\begin{align*}
\min_{x' \in \BR^d_x} \max_{y' \in \BR^{d_y}} g(x,y) \triangleq f(x,y) + L \Vert x' - x \Vert^2 - L \Vert y - y' \Vert^2.    
\end{align*}
For any $w=(u,v)$, we have
\begin{align*}
    \Vert F(w) \Vert &\le  \Vert F(z^+) \Vert +\Vert F(z^+) - F(w) \Vert.
\end{align*}
Since it holds that $F(z^+) = F_{2L}(z)$ and we already have $\Vert F(z^+) \Vert \le\epsilon$, the smoothness of~$F$ means 
\begin{align}\label{cond:final}
\BE\Vert w - z^+ \Vert \le \frac{\epsilon}{L},
\end{align}
which can make sure that $w=(u,v)$ is a $2\eps$-stationary point of $f(x,y)$ in expectation. 
We can verify the condition number of $g(x,y)$ is $\Theta(1)$, which means finding $w=(u,b)$ that satisfies (\ref{cond:final}) can be finished within $\fO(\log(\epsilon^{-1})+ \sigma^2 \epsilon^{-2})$ SFO complexity by running Epoch-SEG (Algorithm~\ref{alg:Epoch-SEG}) and following the result of Lemma~\ref{lem:Epoch-SEG}. 
\end{proof}

\section{The Proofs in Section~\ref{sec:lower-bound}} \label{apx:func}


\citet{foster2019complexity} showed that the lower bound complexity for finding a point with small gradient can be decomposed as the statistical complexity given by stochastic oracle and the optimization complexity with deterministic oracle. 
We following their ideas to construct our lower bounds for stochastic minimax optimization.

\subsection{The Proof of Theorem \ref{thm:lower-SC}}

\begin{proof} 
We first consider the statistical complexity. There exists an $L$-smooth and convex function $f_{\rm sample}$ such that any $\fA$ needs at least an SFO complexity of $\Omega(\sigma^2 \epsilon^{-2} \log(L \epsilon^{-1}))$ to find its $\epsilon$-stationary point~\cite[Theorem 2]{foster2019complexity}. 
The worst-case convex function for the sample complexity~\cite{foster2019complexity} is given by its first-order derivative
\begin{align*}
f_{\rm sample}^{\prime}\left(x; Z_{t}\right)=
\begin{cases}
-2 \epsilon, & x<0, \\
2 \epsilon, & x \geq D, \\
\dfrac{x-a_{j}}{D / N} \sigma Z_{t, j+1}+\left(1-\dfrac{x-a_{j}}{D / N}\right) \sigma Z_{t, j} & x \in\left[a_{j}, a_{j+1}\right) \text { for some } j<N,
\end{cases}
\end{align*}
where $Z_{t,j} \in \{-1,1\}$ is drawn form the following distribution
\begin{align*}
\BP\left(Z_{t, j}=1\right)=
\begin{cases}
\frac{1}{2}-p, & j \leq j^{*}, \\
\frac{1}{2}+p, & j>j^*.
\end{cases}
\end{align*}
In above equations, the values of $(a_1,\cdots,a_N)$ and $j^*$ is given by the lower bound for noisy binary search problem \cite{foster2019complexity}. 
Setting $N = LD / (4 \epsilon)$ with $D= \Vert x^* - x_0 \Vert$, we obtain the function \begin{align*}
    F(x) = \mathbb{E}_{Z_t}[f_{\rm sample}(x;Z_t)],
\end{align*} 
which is $L$-smooth and convex; and leads to a lower bound 
\begin{align} \label{lower:order}
    \Omega(\sigma^2 \epsilon^{-2} \log (L D \epsilon^{-1}))
\end{align} 
for finding an $\epsilon$-stationary point of $F(x)$. 
For minimax optimization, we construct the function
\begin{align*}
    H(x,y) = F(x) - F(y),
\end{align*}
that is $L$-smooth and convex-concave. Naturally, it provide the lower bound  (\ref{lower:order}) for finding an $\epsilon$-stationary point for $H(x,y)$.

Then we consider the optimization complexity.
We consider the strongly-convex-strongly-concave function~\cite{luo2021near} as follows
\begin{align*}
    f_{\rm SCSC}(x,y) = \frac{\lambda' r}{2} \Vert x \Vert^2 + \lambda' x^\top (B y - c) - \frac{\lambda' r}{2} \Vert y  \Vert^2,
\end{align*}
where 
\begin{align*}
B = \begin{bmatrix}
    1 & & & & \\
    -1 & 1 & & & \\
    & \ddots & \ddots & & \\
    & & -1 & 1 & \\
    & & & -1 & \sqrt{r \omega}
\end{bmatrix} \in \BR^{d\times d}, \quad
c = \begin{bmatrix}
\omega \\ 0 \\ \vdots \\ 0
\end{bmatrix}
\quad\text{and}\quad
\omega = \frac{\sqrt{r^2+4} - r}{2}.    
\end{align*}
By setting
\begin{align*}
r=\sqrt{\frac{8}{L^{2} / \lambda'^{2}-2}}, \quad \lambda'=\frac{\lambda}{r}
\quad \text{and}\quad 
d=\left\lfloor\frac{1}{r} \log \left(\frac{1}{2 \varepsilon}\right)\right\rfloor-4,
\end{align*}
we obtain the special case of the lower bound in Theorem 1 of \citet{luo2021near} with $n = 1$. It is shown that~\cite{luo2021near} the function $f_{\rm SCSC}(x,y)$ is $L$-smooth, $\lambda$-strongly-convex-$\lambda$-strongly-concave and any first-order algorithm requires at least $\Omega(\kappa\log (1/\epsilon))$ number of gradient calls to obtain a point $z$ such that $\Vert z - z^* \Vert\leq\epsilon$.
Since it holds that $\Vert \nabla f(z) \Vert \ge \lambda \Vert z - z^* \Vert$, any $\fA$ needs at least an SFO complexity of 
\begin{align}\label{lower:order2}
\Omega(\kappa \log(\lambda \epsilon^{-1}))    
\end{align}
to find an $\epsilon$-stationary point of $f_{\rm SCSC}(x,y)$.

Combining the statistical complexity of (\ref{lower:order}) and the optimization complexity of (\ref{lower:order2}) completes the proof.
\end{proof}



\begin{algorithm*}[t]  
\caption{SEAG $(f, z_0, \eta, T)$} 
\begin{algorithmic}[1] \label{alg:prac-EAG}
\STATE \textbf{for} $t = 0,1,\cdots,T-1$ \textbf{do} \\[0.1cm]
\STATE \quad $\xi_i \leftarrow $ a random index \\[0.1cm]
\STATE \quad $z_{t+1/2} \leftarrow z_t - \left( 1 - \frac{1}{t+1} \right) \eta  F(z_t;\xi_i) + \frac{1}{t+1} (z_0 - z_t)$ \\[0.1cm]
\STATE \quad $\xi_j \leftarrow $ a random index \\[0.1cm]
\STATE \quad $z_{t+1} \leftarrow  z_t - \eta F(z_{t+1/2};\xi_j) +\frac{1}{t+1} (z_0 - z_t)$ \\[0.1cm]
\STATE \textbf{end for} \\[0.1cm]
\STATE \textbf{return} $z_T$ 
\end{algorithmic}
\end{algorithm*}

\begin{algorithm*}[t]  
\caption{PDHG $(f, x_0, y_0, ,\eta, T)$} 
\begin{algorithmic}[1] \label{alg:prac-PDHG}
\STATE $\bar x_0 = x_0$, $\bar y_0 = y_0$ \\ [0.1cm] 
\STATE $\xi_y^0 \leftarrow$ a random index \\[0.1cm]
\STATE $s_0 = \nabla_y f(x_0,y_0;\xi_y^0)$ \\ [0.1cm]
\STATE \textbf{for} $t = 0,1,\cdots,T-1$ \textbf{do} \\[0.1cm]
\STATE \quad $y_{t+1} \rightarrow y_t - \eta s_t$ \\ [0.1cm]
\STATE \quad $ \xi_x^t \leftarrow$ a random index \\[0.1cm]
\STATE \quad $x_{t+1}  \leftarrow x_t - \eta \nabla_x f(x_t,y_{t+1}; \xi_x^t)$ \\[0.1cm]
\STATE \quad $\xi_y^t \leftarrow$ a random index \\[0.1cm] 
\STATE \quad $s_{t+1} \leftarrow \frac{t+2}{t+1} \nabla_y f(x_t,y_t; \xi_y^t) - \frac{t}{t+1} \nabla_y f(x_t,y_t; \xi_y^{t-1})$ \\[0.1cm]
\STATE \quad $\bar x_{t+1} \leftarrow \frac{t-1}{t+1} \bar x_t + \frac{2}{t+1} x_{t+1}$ \\[0.1cm] 
\STATE \quad $\bar y_{t+1} \leftarrow \frac{t-1}{t+1} \bar y_t + \frac{2}{t+1} y_{t+1}$ \\[0.1cm] 
\STATE \textbf{end for} \\[0.1cm]
\STATE \textbf{return} $(\bar x_T,\bar y_T)$ 
\end{algorithmic}
\end{algorithm*}

\begin{algorithm*}[ht]  
\caption{A Single-Loop Variant of RAIN $(f, z_0, \eta, T, \lambda, \gamma)$} 
\begin{algorithmic}[1] \label{alg:prac-RAIN}
\STATE \textbf{for} $s = 0,1,\cdots,T-1$ \textbf{do} \\[0.1cm]
\STATE \quad $\xi_i \leftarrow $ a random index \\[0.1cm]
\STATE \quad $z_{t+1/2} \leftarrow z_t - \eta ( F(z_t;\xi_i) + \lambda \gamma \sum_{j=0}^{t-1} (1+\gamma)^j (z_t - z_j  ) )$ \\[0.1cm]
\STATE \quad $\xi_j \leftarrow $ a random index \\[0.1cm]
\STATE \quad $z_{t+1} \leftarrow z_t - \eta (F(z_{t+1/2};\xi_j) + \lambda \gamma \sum_{j=0}^{t-1} (1+\gamma)^j (z_{t+1/2} - z_{j}  ) )$ \\[0.1cm]
\STATE \textbf{end for} \\[0.1cm]
\STATE \textbf{return} $z_T$ 
\end{algorithmic}
\end{algorithm*}

\subsection{The Proof of Theorem \ref{thm:lower-C}}

\begin{proof}
The statistical complexity can be obtained by following the proof of Theorem~\ref{thm:lower-SC}.
For the optimization complexity, Theorem 3 of \citet{yoon2021accelerated} provide the convex-concave function of the form
\begin{align*}
    f_{\rm CC}(x,y) = (x- x^*)^\top A (y -x^*),
\end{align*}
where $A \in \BR^{d\times d}$ is symmetric and $x^*$ is the component of a stationary point $z^*=(x^*,y^*)$.
\citet{yoon2021accelerated} showed that there exist some $A$ and $x^*$ such that $0 \preceq A \preceq L$ and finding an $\epsilon$-stationary point of the corresponding $f_{\rm CC}(x,y)$ requires at least $\Omega(L {\color{purple} D}\epsilon^{-1})$ gradient calls.
We complete the proof by combining the optimization complexity in Theorem~\ref{thm:lower-SC}.
\end{proof}

\subsection{The Proof of Corollary \ref{cor:lower-NC}}

\begin{proof}
Since the gradient operator of convex-concave function is monotone, it is also negative comonotone. Hence, the lower bound in Theorem~\ref{thm:lower-C} is also applicable for negative comonotone setting and the result of this corollary is obtained.
\end{proof}

\subsection{The Proof of Corollary \ref{cor:lower-ID}}

\begin{proof}
For $\tau \ge 2L$, the assumptions of $L$-smooth and $\alpha$-strongly-convex-$\alpha$-strongly-concave on $f(x,y)$ directly means the function is $(\tau,\alpha)$-intersection-dominant.
Hence, the lower bound in Theorem~\ref{thm:lower-SC} is also applicable for  $(\tau,\alpha)$-intersection-dominant setting and the result of this corollary is obtained.
\end{proof}

\section{Details for Numerical Experiments} \label{apx:exp}

We give the details for the implementation of the algorithms in our experiments. 
\subsection{The Concave-Concave Case}

The implementations of algorithms in convex-concave case:
\begin{itemize}
\item For ordinary stochastic extragradient method (Algorithm~\ref{alg:SEG}, SEG), we replace the output from uniform sampling in all the history with the point in the last iteration for better performance.
\item For regularized SEG (R-SEG), we use the regularization trick of \citet{nesterov2012make} for minimax optimization, that is using SEG to solve the regularized minimax problem 
\begin{align*}
\min_{x\in\BR^{d_x}}\max_{y\in\BR^{d_y}} g(x,y) \triangleq f(x,y) + \frac{\lambda}{2} \Vert x - x_0 \Vert^2 - \frac{\lambda}{2} \Vert y - y_0 \Vert^2
\end{align*}
for some small $\lambda$. 
\item
For stochastic extra anchor gradient (SEAG), we follow the algorithm by \cite{lee2021fast}, which is described in Algorithm \ref{alg:prac-EAG}.
\item {For primal-dual hybrid gradient (PDHG), we follow the Algorithm 1 by \citet{zhao2022accelerated} with $\alpha_t = \tau_t = \eta$, which is described in Algorithm \ref{alg:prac-PDHG}.} 
\item For RAIN (Algorithm \ref{alg:RAIN}), we directly set $N_s = 1$ and $K_s =0$. {
This simplified variant yields a single-loop implementation described in Algorithm \ref{alg:prac-RAIN}, and we observe it has good performce in practice.}
\end{itemize}

{\subsection{The Nonconcave-Nonconcave Case}

The implementations of algorithms in nonconvex-nonconcave case:
\begin{itemize}
    \item For SEG${}^+$, we follow Equation (EG${}_p^+$) by \citet{diakonikolas2021efficient} with $\beta = 1/2$ and $\alpha_k = \eta$,
    by but output the point in the last iteration for better performance. 
    \item For SFEG, we follow Algorithm 1 by~\citep{lee2021fast} by replacing the exact gradient with the stochastic gradient.
    \item For RAIN${}^{++}$, we follow the setting of Algorithm \ref{alg:prac-RAIN} ($N_s=1 $ and $K_s = 0$) for the subroutine RAIN. 
    For the steps of MLMC, we set $N =1$, $M=1$ and $K=1$.  
    This simplified variant is easy to implement and performs well in practice.
\end{itemize}
}

\subsection{Hyperparameter Selection}
\noindent For each algorithm, we tune the parameters $\eta$ from $\{0.005, 0.01, 0.05, 0.1, 1, 5, 10\}$, $\lambda$ from $\{0.001, 0.01, 0.1, 1\}$ and $\gamma$ from $\{0.001, 0.01, 0.1, 1\}$ and reports the best run.

\section{The Merged Algorithms for Easy Reference}
\label{apx:complete-alg}

\begin{algorithm*}[t]  
\caption{RAIN $(z_0, \lambda, L,\{N_s\}_{s=0}^{S-1}, \{K_s \}_{s=0}^{S-1}, \gamma)$} 
\begin{algorithmic}[1] \label{alg:complete-RAIN}
\STATE $\lambda_0 = \gamma \lambda,  S = \lfloor \log_{(1+\gamma)}(L/\lambda) \rfloor$ \\[0.1cm]
\STATE \textbf{for} $s = 0,1,\cdots, S-1$ \\[0.1cm]
\STATE \quad  $z_{s,0} \leftarrow z_{s}$ \\[0.1cm]
\STATE \quad \textbf{for} $k = 0,1,\cdots, N_s-1$ \\ [0.1cm]
\STATE \quad \quad $z_{s,k,0} \leftarrow z_{s,k}$, $\eta \leftarrow \frac{1}{8L}$, $T \leftarrow \frac{16 L}{\lambda_s}$ \\[0.1cm]
\STATE \quad \quad \textbf{for} $t = 0,1,\cdots, T-1$\\ [0.1cm]
\STATE \quad \quad \quad $z_{s,k,t+1/2} = z_{s,k,t} - \eta (F(z_{s,k,t};\xi_{s,k,t}) + \sum_{j=0}^{s-1} \lambda_j (z_{s,k,t} - z_j  ) )$ \\[0.1cm]
\STATE \quad \quad \quad $z_{s,k,t+1} = z_{s,k,t} - \eta (F(z_{s,k,t+1/2};\xi_{s,k,t+1/2}) + \sum_{j=0}^{s-1} \lambda_j (z_{s,k,t+1/2} - z_j  ) )$ \\[0.1cm]
\STATE \quad \quad \textbf{end for} \\[0.1cm]
\STATE \quad \quad $z_{s,k+1} \leftarrow$ uniformly samples from $\{z_{s,k,t+1/2} \}_{t=0}^{T-1}$ \\[0.1cm]
\STATE \quad \textbf{end for} \\[0.1cm]
\STATE \quad \textbf{for} $k = N_s,N_s+1,\cdots, N_s+K_s-1$ \\ [0.1cm]
\STATE \quad \quad $z_{s,k,0} \leftarrow z_{s,k}$, $\eta \leftarrow \frac{1}{2^{k-N_s+4} \cdot L}$, $T \leftarrow \frac{2^{k-N_s+6} \cdot}{\lambda_s}$ \\[0.1cm]
\STATE \quad \quad \textbf{for} $t = 0,1,\cdots, T-1$\\ [0.1cm]
\STATE \quad \quad \quad $z_{s,k,t+1/2} = z_{s,k,t} - \eta (F(z_{s,k,t}; \xi_{s,k,t}) +  \sum_{j=0}^{s-1} \lambda_j (z_{s,k,t} - z_j  ) )$ \\[0.1cm]
\STATE \quad \quad \quad $z_{s,k,t+1} = z_{s,k,t} - \eta (F(z_{s,k,t+1/2}; \xi_{s,k,t+1/2}) + \sum_{j=0}^{s-1} \lambda_j (z_{s,k,t+1/2} - z_j ) )$ \\[0.1cm]
\STATE \quad \quad \textbf{end for} \\[0.1cm]
\STATE \quad \quad $z_{s,k+1} \leftarrow$ uniformly samples from $\{z_{s,k,t+1/2} \}_{t=0}^{T-1}$ \\[0.1cm]
\STATE \quad \textbf{end for} \\[0.1cm]
\STATE \quad $z_{s+1} \leftarrow z_{s, N_s+K_s}$ \\[0.1cm]
\STATE \quad $\lambda_{s+1} \leftarrow  (1+ \gamma)\lambda_s$ \\[0.1cm]
\STATE \textbf{end for} \\[0.1cm]
\STATE \textbf{return} $z_{S}$ 
\end{algorithmic}
\end{algorithm*}

\begin{algorithm*}[t]  
\caption{RAIN${}^{++}$ $(z_0, \lambda, L,\{N_s\}_{s=0}^{S-1}, \{K_s \}_{s=0}^{S-1}, \gamma, N, K, M)$} 
\begin{algorithmic}[1] \label{alg:complete-RAIN++}
\STATE $\lambda_0 = \gamma \lambda,  S = \lfloor \log_{(1+\gamma)}(3 L/\lambda) \rfloor$ \\[0.1cm]
\STATE \textbf{for} $s = 0,1,\cdots, S-1$ \\[0.1cm]
\STATE \quad  $z_{s,0} \leftarrow z_{s}$, $w_{s,0} \leftarrow w_s$ \\[0.1cm]
\STATE \quad \textbf{for} $k = 0,1,\cdots, N_s-1$ \\ [0.1cm]
\STATE \quad \quad $z_{s,k,0} \leftarrow z_{s,k}$, 
$w_{s,k,-1/2} \leftarrow w_{s,k}$,
$\eta \leftarrow \frac{1}{24 L}$, $T \leftarrow \frac{48 L}{\lambda_s}$ \\[0.1cm]
\STATE \quad \quad \textbf{for} $t = 0,1,\cdots, T-1$\\ [0.1cm]
\STATE \quad \quad \quad $ (\hat F_{s,k,t} , w_{s,k,t}) = \text{EnvGradEst}(z_{s,k,t}, w_{s,k,t-1/2}, L, N,K,M) $ \\[0.1cm]
\STATE \quad \quad \quad $z_{s,k,t+1/2} = z_{s,k,t} - \eta (\hat F_{s,k,t} + \sum_{j=0}^{s-1} \lambda_j (z_{s,k,t} - z_j  ) )$ \\[0.1cm]
\STATE \quad \quad \quad $ (\hat F_{s,k,t+1/2} , w_{s,k,t+1/2}) = \text{EnvGradEst}(z_{s,k,t}, w_{s,k,t}, L, N,K,M) $ \\[0.1cm]
\STATE \quad \quad \quad $z_{s,k,t+1} = z_{s,k,t} - \eta (\hat F_{s,k,t+1/2} + \sum_{j=0}^{s-1} \lambda_j (z_{s,k,t+1/2} - z_j  ) )$ \\[0.1cm]
\STATE \quad \quad \textbf{end for} \\[0.1cm]
\STATE \quad \quad Draw $J \sim {\rm Unif}([T])$ \\[0.1cm]
\STATE \quad \quad $z_{s,k+1} \leftarrow z_{s,k,J+1/2}$, $w_{s,k+1} \leftarrow w_{s,k,J+1/2}$  \\[0.1cm]
\STATE \quad \textbf{end for} \\[0.1cm]
\STATE \quad \textbf{for} $k = N_s,N_s+1,\cdots, N_s+K_s-1$ \\ [0.1cm]
\STATE \quad \quad $z_{s,k,0} \leftarrow z_{s,k}$, $w_{s,k,-1/2} \leftarrow w_{s,k}$, $\eta \leftarrow \frac{1}{2^{k-N_s+4} \cdot 3 L}$, $T \leftarrow \frac{2^{k-N_s+6} \cdot 3L}{\lambda_s}$ \\[0.1cm]
\STATE \quad \quad \textbf{for} $t = 0,1,\cdots, T-1$\\ [0.1cm]
\STATE \quad \quad \quad $ (\hat F_{s,k,t} ,w_{s,k,t}) = \text{EnvGradEst}(z_{s,k,t}, w_{s,k,t-1/2}, L, N,K,M) $ \\[0.1cm]
\STATE \quad \quad \quad $z_{s,k,t+1/2} = z_{s,k,t} - \eta (\hat F_{s,k,t} +  \sum_{j=0}^{s-1} \lambda_j (z_{s,k,t} - z_j  ) )$ \\[0.1cm]
\STATE \quad \quad \quad $ (\hat F_{s,k,t+1/2} , w_{s,k,t+1/2}) = \text{EnvGradEst}(z_{s,k,t}, w_{s,k,t}, L, N,K,M) $ \\[0.1cm]
\STATE \quad \quad \quad $z_{s,k,t+1} = z_{s,k,t} - \eta (\hat F_{s,k,t+1/2} + \sum_{j=0}^{s-1} \lambda_j (z_{s,k,t+1/2} - z_j ) )$ \\[0.1cm]
\STATE \quad \quad \textbf{end for} \\[0.1cm]
\STATE \quad \quad Draw $J \sim {\rm Unif}([T])$ \\[0.1cm]
\STATE \quad \quad $z_{s,k+1} \leftarrow z_{s,k,J+1/2}$, $w_{s,k+1} \leftarrow w_{s,k,J+1/2}$  \\[0.1cm]
\STATE \quad \textbf{end for} \\[0.1cm]
\STATE \quad $z_{s+1} \leftarrow z_{s, N_s+K_s}$, $ w_{s+1} \leftarrow w_{s, N_s+K_s}$ \\[0.1cm]
\STATE \quad $\lambda_{s+1} \leftarrow  (1+ \gamma)\lambda_s$ \\[0.1cm]
\STATE \textbf{end for} \\[0.1cm]
\STATE \textbf{return} $z_{S}$ 
\end{algorithmic}
\end{algorithm*}

\begin{algorithm*}[t]  
\caption{EnvGradEst$(\bar z, z_0, L, N, K, M)$} 
\begin{algorithmic}[1] \label{alg:Saddle-Grad-Est}
\STATE \textbf{for} $m = 0,1,\cdots,M-1$ \textbf{do} \\[0.1cm]
\STATE \quad draw $J \sim {\rm Geom}\left(1/2\right)$ \\[0.1cm]
\STATE \quad \textbf{for} $k=0,1,\cdots, N-1$ \\[0.1cm]
\STATE \quad \quad $z_{m,k,0} \leftarrow z_{0}$, $\eta \leftarrow \frac{1}{12L}$, $T \leftarrow 24$ \\ [0.1cm]
\STATE \quad \quad \textbf{for} $t = 0,1,\cdots,T-1$ \\[0.1cm]
\STATE \quad \quad \quad$z_{m,k,t+1/2} \leftarrow z_{m,k,t} -\eta (F(z_{m,k,t};\xi_{m,k,t}) + 2L(z_{m,k,t} - \bar z)) $ \\[0.1cm]
\quad \quad \quad $z_{m,k,t+1} \leftarrow z_{m,k,t} -\eta (F(z_{m,k,t+1/2};\xi_{m,k,t+1/2}) + 2L(z_{m,k,t+1/2} - \bar z)) $ \\[0.1cm] 
\quad \quad \textbf{end for} \\[0.1cm] 
\STATE \quad \quad $z_{m,k+1} \leftarrow$ uniformly samples from $\{z_{m,k,t+1/2} \}_{t=0}^{T-1}$ \\[0.1cm]
\STATE \quad \textbf{end for} \\[0.1cm]
\STATE \quad \textbf{for} $k = N,N+1,\cdots, N+J-1$ \\ [0.1cm]
\STATE \quad \quad $z_{m,k,0} \leftarrow z_{m,k}$, $\eta \leftarrow \frac{1}{2^{k-N+3} \cdot 3L}$, $T \leftarrow 2^{k-N+5} \cdot 3L$\\ [0.1cm]
\STATE \quad \quad \textbf{for} $t = 0,1,\cdots,T-1$ \\[0.1cm]
\STATE \quad \quad \quad$z_{m,k,t+1/2} \leftarrow z_{m,k,t} -\eta (F(z_{m,k,t};\xi_{m,k,t}) + 2L(z_{m,k,t} - \bar z)) $ \\[0.1cm]
\quad \quad \quad $z_{m,k,t+1} \leftarrow z_{m,k,t} -\eta (F(z_{m,k,t+1/2};\xi_{m,k,t+1/2}) + 2L(z_{m,k,t+1/2} - \bar z)) $ \\[0.1cm] 
\quad \quad \textbf{end for} \\[0.1cm] 
\STATE \quad \quad $z_{m,k+1} \leftarrow$ uniformly samples from $\{z_{m,k,t+1/2} \}_{t=0}^{T-1}$ \\[0.1cm]
\STATE \quad \textbf{end for} \\[0.1cm]
\STATE \quad $\hat z_m =  z_{m,N} + 2^J (z_{m,N+J} - z_{m,N+J-1}) \mathbb{I}\,[J \le K]$ \\[0.1cm]
\STATE \textbf{end for} \\[0.1cm]
\STATE $\hat z^+ \leftarrow \frac{1}{M} \sum_{m=0}^{M-1} \hat z_m$, $\hat F = 2L(\bar z - \hat z^+)$ \\[0.1cm]
\STATE \textbf{return} $(\hat F, \hat z^+)$
\end{algorithmic}
\end{algorithm*}

{In the main text, we present our algorithm by nested functions to facilitate the theoretical analysis. 
In this section, we write all steps of RAIN in Algorithm \ref{alg:RAIN} without the presentation of any subroutine call, which is easy to follow for readers who are interested in the implementation.
We also provided a merged presentation for RAIN$^{++}$ in Algorithm~\ref{alg:complete-RAIN++}.
Since the steps of RAIN$^{++}$ is indeed complicated, 
it still includes one subroutine (Algorithm \ref{alg:Saddle-Grad-Est}) to present the evaluation for the gradient of the saddle envelope.
Note that EnvGradEst (Algorithm \ref{alg:Saddle-Grad-Est}) achieves a nearly unbiased gradient estimator of the saddle envelope and the initial point in the next subroutine call.
In the main text, EnvGradEst is presented by applying Epoch-SEG${}^+$ on the sub-problem 
\begin{align*}
\min_{x\in\BR^{d_x}}\max_{y\in\BR^{d_y}}
g_{t,s,k}(x,y): = f(x_{t,s,k},y_{t,s,k}) + L \Vert x - x_{t,s,k} \Vert^2 - L \Vert y - y_{t,s,k} \Vert^2      
\end{align*}
to output $\hat F_{s,k,t} = 2 L (z_{s,k,t} - w_{s,k,t})$ and $w_{s,k,t}$.
We can verify they are equivalent.

}

\newpage

\bibliography{22-1126}

\begin{thebibliography}{37}
\providecommand{\natexlab}[1]{#1}
\providecommand{\url}[1]{\texttt{#1}}
\expandafter\ifx\csname urlstyle\endcsname\relax
  \providecommand{\doi}[1]{doi: #1}\else
  \providecommand{\doi}{doi: \begingroup \urlstyle{rm}\Url}\fi

\bibitem[Alacaoglu and Malitsky(2022)]{alacaoglu2022stochastic}
Ahmet Alacaoglu and Yura Malitsky.
\newblock Stochastic variance reduction for variational inequality methods.
\newblock In \emph{COLT}, 2022.

\bibitem[Allen-Zhu(2018)]{allen2018make}
Zeyuan Allen-Zhu.
\newblock How to make the gradients small stochastically: Even faster convex and nonconvex {SGD}.
\newblock In \emph{NeurIPS}, 2018.

\bibitem[Asi et~al.(2021)Asi, Carmon, Jambulapati, Jin, and Sidford]{asi2021stochastic}
Hilal Asi, Yair Carmon, Arun Jambulapati, Yujia Jin, and Aaron Sidford.
\newblock Stochastic bias-reduced gradient methods.
\newblock In \emph{NeurIPS}, 2021.

\bibitem[Cai et~al.(2022)Cai, Song, Guzm{\'a}n, and Diakonikolas]{cai2022stochastic}
Xufeng Cai, Chaobing Song, Crist{\'o}bal Guzm{\'a}n, and Jelena Diakonikolas.
\newblock A stochastic halpern iteration with variance reduction for stochastic monotone inclusion problems.
\newblock In \emph{NeurIPS}, 2022.

\bibitem[Dai et~al.(2018)Dai, Shaw, He, Li, and Song]{dai2018boosting}
Bo~Dai, Albert Shaw, Niao He, Lihong Li, and Le~Song.
\newblock Boosting the actor with dual critic.
\newblock In \emph{ICLR}, 2018.

\bibitem[Diakonikolas et~al.(2021)Diakonikolas, Daskalakis, and Jordan]{diakonikolas2021efficient}
Jelena Diakonikolas, Constantinos Daskalakis, and Michael~I. Jordan.
\newblock Efficient methods for structured nonconvex-nonconcave min-max optimization.
\newblock In \emph{AISTATS}, 2021.

\bibitem[Foster et~al.(2019)Foster, Sekhari, Shamir, Srebro, Sridharan, and Woodworth]{foster2019complexity}
Dylan~J. Foster, Ayush Sekhari, Ohad Shamir, Nathan Srebro, Karthik Sridharan, and Blake Woodworth.
\newblock The complexity of making the gradient small in stochastic convex optimization.
\newblock In \emph{COLT}, 2019.

\bibitem[Goodfellow et~al.(2014{\natexlab{a}})Goodfellow, Pouget-Abadie, Mirza, Xu, Warde-Farley, Ozair, Courville, and Bengio]{goodfellow2014generative}
Ian Goodfellow, Jean Pouget-Abadie, Mehdi Mirza, Bing Xu, David Warde-Farley, Sherjil Ozair, Aaron Courville, and Yoshua Bengio.
\newblock Generative adversarial nets.
\newblock In \emph{NIPS}, 2014{\natexlab{a}}.

\bibitem[Goodfellow et~al.(2014{\natexlab{b}})Goodfellow, Shlens, and Szegedy]{goodfellow2014explaining}
Ian~J Goodfellow, Jonathon Shlens, and Christian Szegedy.
\newblock Explaining and harnessing adversarial examples.
\newblock \emph{arXiv preprint arXiv:1412.6572}, 2014{\natexlab{b}}.

\bibitem[Grimmer et~al.(2023)Grimmer, Lu, Worah, and Mirrokni]{grimmer2020landscape}
Benjamin Grimmer, Haihao Lu, Pratik Worah, and Vahab Mirrokni.
\newblock The landscape of the proximal point method for nonconvex--nonconcave minimax optimization.
\newblock \emph{Mathematical Programming}, 201\penalty0 (1):\penalty0 373--407, 2023.

\bibitem[Guo et~al.(2020)Guo, Yuan, Yan, and Yang]{guo2020fast}
Zhishuai Guo, Zhuoning Yuan, Yan Yan, and Tianbao Yang.
\newblock Fast objective \& duality gap convergence for nonconvex-strongly-concave min-max problems.
\newblock \emph{arXiv preprint arXiv:2006.06889}, 2020.

\bibitem[Kovalev and Gasnikov(2022)]{kovalev2022first}
Dmitry Kovalev and Alexander Gasnikov.
\newblock The first optimal algorithm for smooth and strongly-convex-strongly-concave minimax optimization.
\newblock \emph{arXiv preprint arXiv:2205.05653}, 2022.

\bibitem[Lee and Kim(2021)]{lee2021fast}
Sucheol Lee and Donghwan Kim.
\newblock Fast extra gradient methods for smooth structured nonconvex-nonconcave minimax problems.
\newblock In \emph{NeurIPS}, 2021.

\bibitem[Lin et~al.(2020{\natexlab{a}})Lin, Jin, and Jordan]{lin2020gradient}
Tianyi Lin, Chi Jin, and Michael~I. Jordan.
\newblock On gradient descent ascent for nonconvex-concave minimax problems.
\newblock In \emph{ICML}, 2020{\natexlab{a}}.

\bibitem[Lin et~al.(2020{\natexlab{b}})Lin, Jin, and Jordan]{lin2020near}
Tianyi Lin, Chi Jin, and Michael~I. Jordan.
\newblock Near-optimal algorithms for minimax optimization.
\newblock In \emph{COLT}, 2020{\natexlab{b}}.

\bibitem[Liu and Orabona(2022)]{liu2022initialization}
Mingrui Liu and Francesco Orabona.
\newblock On the initialization for convex-concave min-max problems.
\newblock In \emph{International Conference on Algorithmic Learning Theory}, 2022.

\bibitem[Liu et~al.(2019)Liu, Yuan, Ying, and Yang]{liu2019stochastic}
Mingrui Liu, Zhuoning Yuan, Yiming Ying, and Tianbao Yang.
\newblock Stochastic {AUC} maximization with deep neural networks.
\newblock \emph{arXiv preprint arXiv:1908.10831}, 2019.

\bibitem[Liu et~al.(2020)Liu, Zhang, Mroueh, Cui, Ross, Yang, and Das]{liu2020decentralized}
Mingrui Liu, Wei Zhang, Youssef Mroueh, Xiaodong Cui, Jarret Ross, Tianbao Yang, and Payel Das.
\newblock A decentralized parallel algorithm for training generative adversarial nets.
\newblock In \emph{NeurIPS}, 2020.

\bibitem[Liu et~al.(2021)Liu, Rafique, Lin, and Yang]{liu2021first}
Mingrui Liu, Hassan Rafique, Qihang Lin, and Tianbao Yang.
\newblock First-order convergence theory for weakly-convex-weakly-concave min-max problems.
\newblock \emph{Journal of Machine Learning Research}, 2021.

\bibitem[Luo et~al.(2021)Luo, Xie, Zhang, and Zhang]{luo2021near}
Luo Luo, Guangzeng Xie, Tong Zhang, and Zhihua Zhang.
\newblock Near optimal stochastic algorithms for finite-sum unbalanced convex-concave minimax optimization.
\newblock \emph{arXiv preprint arXiv:2106.01761}, 2021.

\bibitem[Madry et~al.(2017)Madry, Makelov, Schmidt, Tsipras, and Vladu]{madry2017towards}
Aleksander Madry, Aleksandar Makelov, Ludwig Schmidt, Dimitris Tsipras, and Adrian Vladu.
\newblock Towards deep learning models resistant to adversarial attacks.
\newblock \emph{arXiv preprint arXiv:1706.06083}, 2017.

\bibitem[Nesterov(2012)]{nesterov2012make}
Yurii Nesterov.
\newblock How to make the gradients small.
\newblock \emph{Optima. Mathematical Optimization Society Newsletter}, 2012.

\bibitem[Omidshafiei et~al.(2017)Omidshafiei, Pazis, Amato, How, and Vian]{omidshafiei2017deep}
Shayegan Omidshafiei, Jason Pazis, Christopher Amato, Jonathan~P. How, and John Vian.
\newblock Deep decentralized multi-task multi-agent reinforcement learning under partial observability.
\newblock In \emph{ICML}, 2017.

\bibitem[Pethick et~al.(2022)Pethick, Patrinos, Fercoq, Cevher{\aa}, et~al.]{pethick2022escaping}
Thomas Pethick, Panagiotis Patrinos, Olivier Fercoq, Volkan Cevher{\aa}, et~al.
\newblock Escaping limit cycles: Global convergence for constrained nonconvex-nonconcave minimax problems.
\newblock In \emph{ICLR}, 2022.

\bibitem[Rockafellar(1970)]{rockafellar1970monotone}
R.~Tyrrell Rockafellar.
\newblock Monotone operators associated with saddle-functions and minimax problems.
\newblock \emph{Nonlinear functional analysis}, 18\penalty0 (part 1):\penalty0 397--407, 1970.

\bibitem[Wai et~al.(2018)Wai, Yang, Wang, and Hong]{wai2018multi}
Hoi-To Wai, Zhuoran Yang, Zhaoran Wang, and Mingyi Hong.
\newblock Multi-agent reinforcement learning via double averaging primal-dual optimization.
\newblock In \emph{NeurIPS}, 2018.

\bibitem[Xian et~al.(2021)Xian, Huang, Zhang, and Huang]{xian2021faster}
Wenhan Xian, Feihu Huang, Yanfu Zhang, and Heng Huang.
\newblock A faster decentralized algorithm for nonconvex minimax problems.
\newblock In \emph{NeurIPS}, 2021.

\bibitem[Xu et~al.(2020)Xu, Wang, Liang, and Poor]{xu2020enhanced}
Tengyu Xu, Zhe Wang, Yingbin Liang, and H~Vincent Poor.
\newblock Enhanced first and zeroth order variance reduced algorithms for min-max optimization.
\newblock \emph{arXiv preprint arXiv:2006.09361}, 2020.

\bibitem[Yang et~al.(2020{\natexlab{a}})Yang, Kiyavash, and He]{yang2020global}
Junchi Yang, Negar Kiyavash, and Niao He.
\newblock Global convergence and variance reduction for a class of nonconvex-nonconcave minimax problems.
\newblock In \emph{NeurIPS}, 2020{\natexlab{a}}.

\bibitem[Yang et~al.(2020{\natexlab{b}})Yang, Zhang, Kiyavash, and He]{yang2020catalyst}
Junchi Yang, Siqi Zhang, Negar Kiyavash, and Niao He.
\newblock A catalyst framework for minimax optimization.
\newblock \emph{NeurIPS}, 2020{\natexlab{b}}.

\bibitem[Yang and Ying(2022)]{yang2022auc}
Tianbao Yang and Yiming Ying.
\newblock Auc maximization in the era of big data and ai: A survey.
\newblock \emph{ACM Computing Surveys}, 2022.

\bibitem[Yoon and Ryu(2021)]{yoon2021accelerated}
TaeHo Yoon and Ernest~K. Ryu.
\newblock Accelerated algorithms for smooth convex-concave minimax problems with $\mathcal{O}(1/k^{2})$ rate on squared gradient norm.
\newblock In \emph{ICML}, 2021.

\bibitem[Yuan et~al.(2021)Yuan, Guo, Xu, Ying, and Yang]{yuan2021federated}
Zhuoning Yuan, Zhishuai Guo, Yi~Xu, Yiming Ying, and Tianbao Yang.
\newblock Federated deep {AUC} maximization for heterogeneous data with a constant communication complexity.
\newblock \emph{arXiv preprint arXiv:2102.04635}, 2021.

\bibitem[Zhang et~al.(2022{\natexlab{a}})Zhang, Thekumparampil, Oh, and He]{zhang2022bring}
Liang Zhang, Kiran~Koshy Thekumparampil, Sewoong Oh, and Niao He.
\newblock Bring your own algorithm for optimal differentially private stochastic minimax optimization.
\newblock \emph{arXiv preprint arXiv:2206.00363}, 2022{\natexlab{a}}.

\bibitem[Zhang et~al.(2021)Zhang, Yang, Guzm{\'a}n, Kiyavash, and He]{zhang2021complexity}
Siqi Zhang, Junchi Yang, Crist{\'o}bal Guzm{\'a}n, Negar Kiyavash, and Niao He.
\newblock The complexity of nonconvex-strongly-concave minimax optimization.
\newblock In \emph{UAI}, 2021.

\bibitem[Zhang et~al.(2022{\natexlab{b}})Zhang, Hu, Zhang, and He]{zhang2022uniform}
Siqi Zhang, Yifan Hu, Liang Zhang, and Niao He.
\newblock Uniform convergence and generalization for nonconvex stochastic minimax problems.
\newblock \emph{arXiv preprint arXiv:2205.14278}, 2022{\natexlab{b}}.

\bibitem[Zhao(2022)]{zhao2022accelerated}
Renbo Zhao.
\newblock Accelerated stochastic algorithms for convex-concave saddle-point problems.
\newblock \emph{Mathematics of Operations Research}, 2022.

\end{thebibliography}

\end{document}